\documentclass{article}

\usepackage{arxiv}
\usepackage{microtype}
\usepackage{amsmath,amssymb,amsfonts,amsthm}
\usepackage{booktabs}
\usepackage{enumitem}
\usepackage{natbib}
\usepackage{hyperref}
\usepackage{xcolor}
\usepackage{tikz}
\usetikzlibrary{arrows.meta,calc,positioning}

\hypersetup{
  pdftitle={Learning in Infinitesimal Non-Compositional Sketches},
  pdfsubject={Tangent categories, sketches, and infinitesimal learning},
  pdfauthor={Sridhar Mahadevan},
  pdfkeywords={tangent categories, sketches, infinitesimal learning, categorical machine learning, diagrammatic backpropagation, connections, Lie brackets, quotient geometry, coalgebras, Kan extensions, reinforcement learning},
  colorlinks=true,
  linkcolor=blue!55!black,
  citecolor=blue!55!black,
  urlcolor=blue!55!black
}

\newtheorem{definition}{Definition}
\newtheorem{axiom}{Axiom}
\newtheorem{proposition}{Proposition}
\newtheorem{theorem}{Theorem}
\newtheorem{conjecture}{Conjecture}
\newtheorem{remark}{Remark}
\newtheorem{condition}{Condition}

\newcommand{\C}{\mathcal{C}}
\newcommand{\D}{\mathcal{D}}

\newcommand{\R}{\mathbb{R}}
\newcommand{\T}{\mathsf{T}}
\newcommand{\LINCS}{\mathsf{LINCS}}
\newcommand{\INC}{\mathsf{INC}}
\newcommand{\Loss}{\mathcal{L}}
\newcommand{\Ob}{\operatorname{Ob}}
\newcommand{\Hom}{\operatorname{Hom}}

\newcommand{\res}{\operatorname{res}}

\newcommand{\Fact}{\operatorname{Fact}}

\newcommand{\Vect}{\operatorname{Vect}}
\newcommand{\Path}{\operatorname{Path}}
\newcommand{\Cl}{\operatorname{Cl}}

\newcommand{\Obs}{\operatorname{Obs}}
\newcommand{\norm}[1]{\left\lVert #1 \right\rVert}

\title{Learning in Infinitesimal Non-Compositional Sketches}

\author{Sridhar Mahadevan\\
Adobe Research and University of Massachusetts, Amherst\\
\texttt{smahadev@adobe.com, mahadeva@umass.edu}}

\date{\today}

\begin{document}
\maketitle

\begin{abstract}
This paper develops a categorical framework -- \emph{Learning in Infinitesimal Non-Compositional Sketches} (\textsc{Lincs}  ) --  that reframes machine learning (ML)  as the repair of non-compositionality: failures of diagrams to factor through their intended quotient sketches lifted to the tangent category setting. ML problems are specified as sketches (graphs with commutativity conditions $\mathcal D$, limit cones $\mathcal L$, and colimit cocones $\mathcal K$), generalizing the usual scalarization of loss functions or vector space assumptions. Non-compositionality is defined purely as failure of a universal factorization problem, not as arithmetic error between the desired and actual predictions.  The central object is Infinitesimal Non-Compositionality (INC): given a learning sketch $\mathbb S=(S,\mathcal D,\mathcal L,\mathcal K)$, whose underlying graph is $S$, and a model $D:J \rightarrow C$, the base defect is the obstruction to factorization $\mbox{Obs}(\mbox{Fact}_{\mathbb S}(D))$.  The tangent lift applies the tangent functor $T$ to obtain $TD:J \rightarrow C$, and INC is the obstruction $\mbox{Obs}(\mbox{Fact}_{\mathbb S}(TD))$ —asking whether infinitesimal perturbations preserve the compositionality constraints.

The paper also introduces Tangent Learning Sketches, which are sketches equipped with Cockett-Cruttwell tangent structure, ensuring that if $D$ is admissible, so is $TD$. This allows the definition of \textsc{Lincs}   categories where learning data includes the pair $(\mbox{Fact}_{\mathbb S}(D),\mbox{Fact}_{\mathbb S}(TD))$. Beyond tangent factorization, the framework admits sketch-specific interaction enrichments.  Lie-bracket closure is the intrinsic antisymmetric specialization; when a connection is declared, the full second-order jet also contains connection-dependent symmetric acceleration.  Parameterized realizations must descend through presentation redundancies, and computational tangent signals are admitted contingently rather than assigned mandatory weight.  The paper defines the INC endofunctor $T_{\mbox{INC}}$, which iterates the tangent lift, producing a tower $D,TD,T^2D, \cdots$ of factorization problems. ML is thereby formulated as the search for a coalgebraic fixed point where successive tangent unfoldings stabilize ($\nu T_{\mbox{INC}}$).  Using the Aczel--Mendler theorem, we prove existence of a final INC coalgebra whenever $T_{\mbox{INC}}$ admits a set-based class realization that creates its final carrier.  Barr's theorem yields a set-sized alternative for accessible realizations and a regular-cardinal bound on the final carrier.  For complete metric realizations in which $T_{\mbox{INC}}$ is contractive, we also prove existence and uniqueness of the stabilized INC behavior, geometric convergence of the exact tower, and a finite-error bound for approximate unfoldings.  A detailed experimental evaluation of \textsc{Lincs} is underway in a number of concrete ML settings, including deep learning, large language models, and reinforcement learning, and is described in companion papers. 
\end{abstract}


\newpage 

\tableofcontents 

\section{Introduction}

The dominant language of machine learning (ML) is optimization: choose a loss function and minimize it  \citep{hastie2009elements,murphy2012machine,SuttonBarto2018}. 
This language is powerful, but it hides a common structure shared across many learning objectives. A loss often {\em scalarizes} a failure of composition. A supervised loss scalarizes a failure of a predictor to compose input with target behavior. A Bellman residual scalarizes a failure of a value function to commute with the Bellman operator. A contrastive loss scalarizes failure of representation invariance across views. A sheaf loss scalarizes failure of local sections to glue. Diagrammatic Backpropagation (DB) made this viewpoint explicit: learning can be driven by failures of diagrammatic compositionality \citep{mahadevan2026catagi}.  A DB loss scalarizes failure of a diagram of computations to commute.

The paradigmatic Transformer model used in large language models (LLMs) provides a suggestive example.  After embedding, a length-$m$
token sequence is represented as a point of a Euclidean array, and the
content-only Transformer map $F$ is permutation equivariant: for a permutation
matrix $P$, the intended sketch contains the commuting condition $F(PX)=P F(X)$. Thus ``scramble the input, then apply the model'' and ``apply the model, then
scramble the output'' are two paths that should agree, and an equivariance loss
scalarizes their failure to compose.  This statement deliberately factors out
positional encodings, causal masks, and other order-dependent structure, which
break or refine the permutation symmetry in an implemented language model.
From the \textsc{Lincs}   viewpoint the scalar loss is only the base-level shadow: one can
also ask whether this equivariance square remains coherent under perturbations
of embeddings, attention maps, and parameters, and whether new obstructions
appear under the iterated lifts $\T^2,\T^3,\ldots$.  This suggests that \textsc{Lincs}  
may expose structural properties of language models that are invisible to the
ordinary scalar objective.
Section~\ref{sec:transformer-equivariance} answers this question explicitly:
exact smooth equivariance propagates through every tangent order, whereas a
nonzero higher obstruction localizes where approximate or deliberately broken
equivariance fails to survive perturbation. 

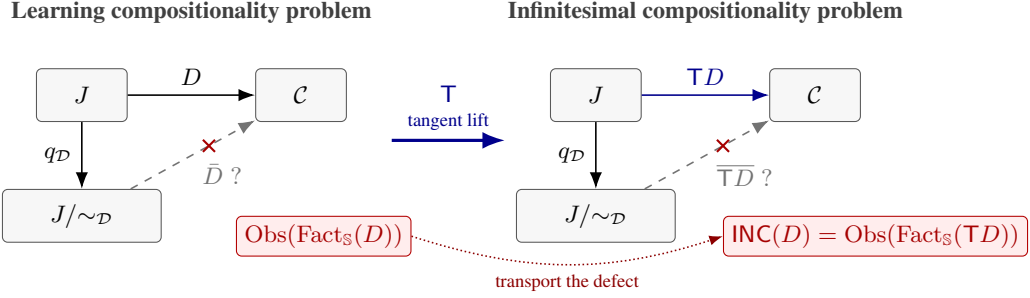
\begin{figure}[t]
\centering
\begin{tikzpicture}[
  font=\small,
  object/.style={draw=black!65, rounded corners=2pt, fill=black!3,
                 minimum width=12mm, minimum height=7mm, inner sep=2pt},
  map/.style={-{Latex[length=2.2mm]}, semithick},
  candidate/.style={-{Latex[length=2.2mm]}, semithick, dashed, black!55},
  lift/.style={-{Latex[length=3mm]}, very thick, blue!55!black},
  obstruction/.style={draw=red!70!black, fill=red!7, rounded corners=2pt,
                      text=red!65!black, inner sep=3pt, align=center},
  paneltitle/.style={font=\small\bfseries, text=black!75}
]
  \node[paneltitle] (base-title) at (0,1.75) {Learning compositionality problem};
  \node[object] (j0) at (-1.45,.65) {$J$};
  \node[object] (c0) at (1.45,.65) {$\C$};
  \node[object, minimum width=21mm] (q0) at (-1.45,-.95)
    {$J/{\sim_{\mathcal D}}$};
  \draw[map] (j0) -- node[above] {$D$} (c0);
  \draw[map] (j0) -- node[left] {$q_{\mathcal D}$} (q0);
  \draw[candidate] (q0) -- node[below right, text=black!55] {$\bar D\;?$} (c0);
  \node[obstruction] (obs0) at (1.75,-1.22)
    {$\Obs\!\left(\Fact_{\mathbb S}(D)\right)$};
  \draw[red!65!black, line width=.8pt]
    ($(q0)!.58!(c0)+(-.08,.08)$) -- ($(q0)!.58!(c0)+(.08,-.08)$);
  \draw[red!65!black, line width=.8pt]
    ($(q0)!.58!(c0)+(-.08,-.08)$) -- ($(q0)!.58!(c0)+(.08,.08)$);

  \draw[lift] (2.65,.05) -- node[above, align=center, text=blue!55!black]
    {$\T$\\[-1pt]{\scriptsize tangent lift}} (4.15,.05);

  \node[paneltitle] (tan-title) at (6.8,1.75) {Infinitesimal compositionality problem};
  \node[object] (j1) at (5.35,.65) {$J$};
  \node[object] (c1) at (8.25,.65) {$\C$};
  \node[object, minimum width=21mm] (q1) at (5.35,-.95)
    {$J/{\sim_{\mathcal D}}$};
  \draw[map, blue!55!black] (j1) -- node[above, text=blue!55!black] {$\T D$} (c1);
  \draw[map] (j1) -- node[left] {$q_{\mathcal D}$} (q1);
  \draw[candidate] (q1) -- node[below right, text=black!55] {$\overline{\T D}\;?$} (c1);
  \node[obstruction, font=\small\bfseries] (obs1) at (9.05,-1.22)
    {$\INC(D)=\Obs\!\left(\Fact_{\mathbb S}(\T D)\right)$};
  \draw[red!65!black, line width=.8pt]
    ($(q1)!.58!(c1)+(-.08,.08)$) -- ($(q1)!.58!(c1)+(.08,-.08)$);
  \draw[red!65!black, line width=.8pt]
    ($(q1)!.58!(c1)+(-.08,-.08)$) -- ($(q1)!.58!(c1)+(.08,.08)$);

  \draw[-{Latex[length=1.8mm]}, red!55!black, densely dotted, semithick]
    (obs0.east) to[bend right=18]
    node[below, text=red!55!black, font=\scriptsize] {transport the defect}
    (obs1.west);
\end{tikzpicture}
\caption{Every learning compositionality problem has a tangent lift.  At
left, a candidate model $D$ may fail to factor through the quotient specified
by the learning sketch $\mathbb S$.  Applying the tangent functor produces the
corresponding factorization problem for $\T D$ without first scalarizing the
base defect.  The obstruction to this lifted factorization is infinitesimal
non-compositionality, $\INC(D)$.}
\label{fig:tangent-lift-inc}
\end{figure}

Using this paradigmatic example as a theme,  \textsc{Lincs}   explores the result of formulating a wide range of  ML
problems, from deep learning to large language models and reinforcement learning, as a categorical compositionality problem lifted by the tangent functor. We can therefore ask not only whether a
base factorization problem is solved, but whether its infinitesimal transport
is coherent. The central thesis is that every learning compositionality problem has a tangent lift.  Figure~\ref{fig:tangent-lift-inc} depicts this passage at the level of the factorization problem itself. The lift defines a new learning signal -- Infinitesimal Non-Compositionality (INC). The usual ML formulation asks: is the scalarized defect small? \textsc{Lincs}   asks: is the defect small, and does its tangent factorization behave correctly?  This paper develops \textsc{Lincs}   as a mathematical framework and illustrates the framework in a number of practical settings, from diagrammatic backpropagation in deep learning \citep{mahadevan2026catagi}, Lie-algebraic LoRA adapters (\textsc{Allora}) for fine-tuning Transformer models \citep{mahadevan2026allora}, Lie-algebroidal skill optimization methods (\textsc{Lasko}) \citep{mahadevan2026lasko}, and latent-confounded causal discovery using Lie-brackets \citep{mahadevan2026bridge}.

\paragraph{Contributions.}
The main contributions of this paper include the following: 
\begin{enumerate}[leftmargin=2em]
  \item It defines Infinitesimal Non-Compositionality (INC) as the tangent lift of diagrammatic factorization failure.
  \item It formulates obstruction localization, a functorial decomposition of global non-compositionality into compatible local learning signals, and shows that this decomposition is preserved by tangent lift.
  \item It formulates \textsc{Lincs}   objects, \textsc{Lincs}   morphisms, and INC categories using tangent-category structure.
  \item It introduces interaction-enriched INC: bracket closure is the intrinsic antisymmetric specialization, while connection-equipped sketches may also retain symmetric acceleration and full second-order jet data.
  \item It distinguishes the categorical tangent lift from its parameterized realization by requiring optimizer fields to descend through presentation quotients or to be supplied with an equivariant horizontal lift.
  \item It gives a coalgebraic view of \textsc{Lincs}  , proves existence of final INC coalgebras under set-based class or accessible set semantics, bounds the set-sized carrier at a regular cardinal, and proves geometric convergence for contractive metric realizations.
  \item It proposes functoriality and universal-completion principles for \textsc{Lincs}   categories.
  \item It relates \textsc{Lincs}   to Kan-invariant learning, tangent categories, infinitesimal causality, Lie-bracket learning, and homotopical repair.
  \item It sketches how GIRL, Differential KET, Diagrammatic Backpropagation, causal Lie-bracket learning, ALLORA, and LASKO instantiate the same core axioms with domain-specific interaction signatures and contingent computational admission.
\end{enumerate}

\section{Related Work}

\paragraph{Sketches and categorical theories.}
The use of sketches to present structured theories goes back to Ehresmann's work on sketches and algebraic structures \citep{ehresmann1968esquisses}. Sketches were later developed as a flexible categorical language for theories and models, notably in the work of Barr and Wells \citep{barr1999category}, Makkai and Par{\'e} \citep{makkai1989accessible}, and Ad{\'a}mek and Rosick{\'y} \citep{adamek1994locally}. Following the Makkai-Par{\'e} presentation, we regard a sketch as a quadruple \((S,\mathcal D,\mathcal L,\mathcal K)\), where \(S\) is an underlying graph, \(\mathcal D\) is a class of commutativity conditions, \(\mathcal L\) is a class of distinguished cones, and \(\mathcal K\) is a class of distinguished cocones. \textsc{Lincs}   uses sketches in this spirit: a learning problem is specified by a graph of formal computations together with the commutative, limiting, and colimiting constraints that a strict learning model would satisfy.

\paragraph{Accessible categories.}
Accessible-category theory provides a natural size-controlled setting for this
sketch semantics.  Under standard smallness hypotheses, many categories of
models presented by sketches are accessible or locally presentable; conversely,
accessible categories admit presentation results in categorical model theory
\citep{makkai1989accessible,adamek1994locally}.  This connection is relevant to
\textsc{Lincs}   because accessibility supplies set-sized families of presentable models
and closure under filtered colimits, offering a plausible technical setting for
constructing categories of learning models and studying tangent or \textsc{Lincs}  
completions without uncontrolled size growth.

\paragraph{Tangent categories.}
Rosick{\'y} introduced abstract tangent functors, and Cockett and Cruttwell developed tangent categories as an axiomatic account of tangent bundle structure \citep{rosicky1984abstract,cockett2014tangent}. The subsequent theory of differential bundles, tangent fibrations, and connections provides a categorical setting in which vector fields, curvature, and transport can be studied without choosing coordinates \citep{cockett2016differential,cockett2017connections}. \textsc{Lincs}   builds directly on this tangent-categorical substrate by treating the Cockett-Cruttwell structure itself as sketch-presentable data acting on learning sketches and their factorization problems.

\paragraph{Weil-algebra semantics for tangent structure.}
Leung gives a functorial classification of the Cockett--Cruttwell axioms in
terms of Weil algebras \citep[Thm.~14.1]{leung2017classifying}.  If $\mathcal
M$ is a category, specifying tangent structure on $\mathcal M$ is equivalent,
up to isomorphism, to specifying a strong monoidal functor
\[
F:(\mathbb N\text{-}\mathsf{Weil}_1,\otimes,\mathbb N)
  \longrightarrow
  (\operatorname{End}(\mathcal M),\circ,1_{\mathcal M})
\]
that preserves the foundational pullbacks and the equalizer encoding the
universality of vertical lift.  The dual-number algebra
$W=\mathbb N[x]/(x^2)$ is sent to the tangent functor $T=F(W)$, while the
structural maps of $W$ are sent to the projection, zero, addition, vertical
lift, and canonical flip.  Thus $\mathbb N\text{-}\mathsf{Weil}_1$ acts as a
classifying theory, or initial tangent structure, rather than merely providing
examples of tangent functors.  This result supplies the precise functorial
semantics behind the tangent action used by \textsc{Lincs}; the additional
\textsc{Lincs} data are the learning sketches, their admissible models, and the
factorization obstructions transported by that action.

\paragraph{Functorial semantics of Lie theory.}
MacAdam's thesis reconnects two strands of Ehresmann's work---sketch theory
and many-object Lie theory---through tangent categories
\citep{macadam2022functorial}.  It presents Lie algebroids by
\emph{involution algebroids}, a tangent-categorical sketch whose models in
smooth manifolds form the category of Lie algebroids, and then identifies
involution algebroids with suitably exact tangent functors out of the
classifying category $\mathsf{Weil}_1$.  It also places Lie differentiation in
a nerve--realization context induced by an infinitesimal groupoid object.
This is directly relevant to \textsc{Lincs}: it supplies a rigorous precedent
for both sketch-presented tangent structure and Lie-algebroidal semantics.
The distinction is that MacAdam classifies infinitesimal geometric structure,
whereas \textsc{Lincs} studies the obstruction to a learning model and its
tangent lift satisfying a chosen sketch.

\paragraph{Compositional learning.}
Several lines of work have used category theory to describe learning compositionally. Backpropagation has been formulated functorially \citep{fong2019backprop}, Universal Decision Learners formulate reinforcement learning and function approximation in terms of Kan-invariant extensions \citep{mahadevan2026udl}, and Kan Extension Transformers use Kan-extension structure to unify attention, diffusion-style completion, and predict-detach self-conditioning \citep{mahadevan2026ket}. Diagrammatic Backpropagation treats noncommuting diagrams as learning signals \citep{mahadevan2026catagi}. \textsc{Lincs}   extends this viewpoint by studying not only base non-compositionality, but also tangent obstructions, sketch-specific interaction profiles and closure, coalgebraic stabilization, and homotopical repair.

\paragraph{Kan Extensions and Universal Decision Learners:}
Universal Decision Learners (UDL) formulate decision learning in terms of Kan-invariant extensions of local decision data \citep{mahadevan2026udl}. In this view, a learner extends partial information along a functor and seeks invariance under a universal extension. Bellman equations, policy evaluation, and function approximation become forms of local-to-global coherence. \textsc{Lincs}   adds a tangent refinement:
\[
\text{Kan invariance in the base category}
\quad\leadsto\quad
\text{Kan invariance in the tangent category}.
\]
If a learned extension is a left or right Kan extension in \(\C\), then \textsc{Lincs}   asks whether its tangent lift satisfies the corresponding universal property in \(\T\C\). This suggests a combined principle: learn by universal extension, then require infinitesimal coherence of the extension.
In this sense, \textsc{Lincs}   can be viewed as tangent Kan-invariant learning.

\section{Learning Sketches and Tangent Learning Sketches}
We introduce the core notions of learning sketches and tangent learning sketches in this section.  Let \(\C\) be a category whose objects are states, representations, hypotheses, local models, sections, policies, or world-model fragments. Its morphisms are computations, transitions, encoders, decoders, interventions, update rules, or gluing maps. A diagram \(D:J\to\C\) expresses a desired compositional relationship. To motivate the sketch-based formalism, note that in a bare category, an arithmetic loss function cannot be encoded by an expression such as \(g\circ f-f\circ g\). \textsc{Lincs}   translates these numerical losses into a categorical notion of non-compositionality,  without assuming additive, metric, Hilbert, or vector-space enrichment.

We begin with a sketch-theoretic formulation. Let
\[
\mathbb S=(S,\mathcal D,\mathcal L,\mathcal K)
\]
be a sketch in the sense of Makkai and Par{\'e}: \(S\) is a graph, \(\mathcal D\) is a class of commutativity conditions in \(S\), \(\mathcal L\) is a class of cones, and \(\mathcal K\) is a class of cocones. Write \(J=\Path(S)\) for the free category on \(S\). The commutativity data \(\mathcal D\) generate a congruence \(\sim_{\mathcal D}\) on the morphisms of \(J\), and hence a quotient functor
\[
q_{\mathcal D}:J\to J/{\sim_{\mathcal D}}.
\]
The cone and cocone data specify the universal constraints a strict model is expected to realize. A diagram \(D:J\to\C\) satisfies the commutativity part of the sketch precisely when it factors through \(q_{\mathcal D}\):
\[
\begin{array}{ccc}
J & \xrightarrow{D} & \C\\
\downarrow q_{\mathcal D} & \nearrow_{\bar D}\\
J/{\sim_{\mathcal D}}
\end{array}
\qquad
D=\bar D q_{\mathcal D}.
\]
Thus non-compositionality is not an arithmetic difference. It is the failure of a universal factorization problem, enriched by the relevant cone and cocone constraints, to have a solution.

\begin{definition}[Learning sketch]
A learning sketch is a quadruple
\[
\mathbb S=(S,\mathcal D,\mathcal L,\mathcal K),
\]
where \(S\) is a graph of formal learning operations, \(\mathcal D\) is a class of path-equations expressing intended commutativity, \(\mathcal L\) is a class of cones expressing limiting constraints, and \(\mathcal K\) is a class of cocones expressing colimiting constraints. If \(J=\Path(S)\), the equations in \(\mathcal D\) generate a quotient \(q_{\mathcal D}:J\to J/{\sim_{\mathcal D}}\).
\end{definition}

The passage from categories to sketches gains expressive power but weakens
automatic transport.  Spivak makes this tradeoff explicit for database
schemas: specified limit and colimit cones can express constraints such as one
table being the product of two others, but the resulting sketch formalism no
longer inherits the full collection of built-in data-migration functors
available for unconstrained categorical schemas \citep{spivak2014lifting}.
The same issue matters for \textsc{Lincs}  .  A map of the underlying learning graphs need
not preserve the designated cones, cocones, admissible models, or their repair
spaces.  Accordingly, transport between learning sketches is additional
structure: a \textsc{Lincs}   morphism or change of sketch must state which constraints
and factorization problems it preserves.  This is why the later functoriality
results impose preservation of learning sketches and quotient factorization
problems rather than deriving it from an arbitrary schema map.

There is also a limitation at the level of presentation itself.  Barr and
Wells show that not every natural category of structured objects and
structure-preserving morphisms is the category of Set-valued models of an
ordinary sketch: for example, certain categories defined by preservation of
subinitial objects, and groups with center-preserving homomorphisms, are not
sketchable in this sense \citep{barr1992limitations}.  They also indicate how
higher-order sketches can recover some structures beyond this boundary.  Thus
the \textsc{Lincs}   definitions are relative to a chosen sketch-presentable class of
learning models and morphisms.  When admissibility or preservation conditions
quantify over structure that an ordinary sketch cannot present, a higher-order,
enriched, or fibrational replacement may be required.  The tangent action on
factorization problems is already formulated in a way that permits such a
refinement.

\begin{definition}[Model of a learning sketch]
Given a category \(\C\), a candidate model of a learning sketch \(\mathbb S\) in \(\C\) is a functor \(D:J\to\C\), where \(J=\Path(S)\). It is strict, or compositional, when there exists a functor \(\bar D:J/{\sim_{\mathcal D}}\to\C\) such that \(D=\bar Dq_{\mathcal D}\), and the images of the distinguished cones and cocones satisfy the universal requirements specified by \(\mathcal L\) and \(\mathcal K\).
\end{definition}

\begin{definition}[Non-compositionality as factorization failure]
The non-compositionality of a model \(D:J\to\C\) of a learning sketch is the factorization problem
\[
\Fact_{\mathbb S}(D)=\{\bar D:J/{\sim_{\mathcal D}}\to\C\mid D=\bar Dq_{\mathcal D}\text{ and }\bar D\text{ realizes }\mathcal L,\mathcal K\}.
\]
The model is compositional exactly when \(\Fact_{\mathbb S}(D)\) is inhabited. When only the commutativity component is under discussion, we write this as \(\Fact_{q_{\mathcal D}}(D)\).
\end{definition}

\begin{definition}[Tangent sketch]
A tangent sketch is a sketch
\[
\mathbb T=(S_T,\mathcal D_T,\mathcal L_T,\mathcal K_T)
\]
whose graph contains formal operations for a tangent endofunctor \(T\), projection, zero, addition, vertical lift, and canonical flip, and whose commutativity, cone, and cocone data present the Cockett-Cruttwell tangent-category axioms. A model of \(\mathbb T\) in a category \(\C\) is precisely a choice of tangent-category structure on \(\C\), up to the level of strictness encoded by the sketch.
\end{definition}

\begin{definition}[Tangent learning sketch]
A tangent learning sketch is a learning sketch \(\mathbb S=(S,\mathcal D,\mathcal L,\mathcal K)\) equipped with an action of a tangent sketch on its factorization problem. Concretely, if \(J=\Path(S)\) and \(D:J\to\C\) is admissible in a model \((\C,\T)\) of the tangent sketch, then \(\T D:J\to\C\) is again admissible and the base factorization problem
\[
\Fact_{\mathbb S}(D)
\]
is transported to the tangent factorization problem
\[
\Fact_{\mathbb S}(\T D).
\]
Thus a tangent learning sketch is not merely a learning sketch interpreted in a tangent category; it is a sketch whose intended learning constraints are stable under the Cockett-Cruttwell tangent structure.
\end{definition}

\begin{remark}[Infinitesimal objects and representable tangent semantics]
Leung's classification identifies general tangent structure on $\C$ with a
strong monoidal functor from $\mathbb N\text{-}\mathsf{Weil}_1$ to
$\operatorname{End}(\C)$ preserving the designated foundational pullbacks and
vertical-lift equalizer \citep[Thm.~14.1]{leung2017classifying}.  MacAdam gives
a complementary representable realization of this functorial semantics and
uses it in the semantics of Lie theory \citep{macadam2022functorial}.
Following Cockett and Cruttwell, an
infinitesimal object in a symmetric monoidal category is an object $\mathbb D$
equipped with zero, augmentation, multiplication, and coaddition maps satisfying
the stated (co)universality axioms.  Equivalently, in the symmetric monoidal
closed setting it determines a strict symmetric monoidal functor
\[
\mathbb D^{(-)}:\mathsf{Weil}_1\longrightarrow\C.
\]
It induces the representable tangent functor $T=[\mathbb D,-]$ on $\C$ and the
dual tangent functor $T=\mathbb D\otimes(-)$ on $\C^{\mathrm{op}}$.  More
generally, tangent structure on $\C$ is equivalently expressed by a
$\mathsf{Weil}_1$ sketch action
\[
\mathsf{Weil}_1\times\C\longrightarrow\C.
\]
Consequently, a representable \textsc{Lincs} model admits a concrete reading:
$TD$ probes the learning diagram $D$ by the infinitesimal object $\mathbb D$,
and $\operatorname{INC}(D)$ asks whether this infinitesimal probe still
realizes the factorization, limit, and colimit constraints of the learning
sketch.  MacAdam's infinitesimal object is therefore not itself an INC
obstruction; it represents the tangent structure with which the obstruction is
formed.
\end{remark}

\begin{remark}[Lie-algebroidal realization of a tangent learning sketch]
When a tangent learning sketch is realized over smooth state or parameter
spaces, it can admit a more structured, Lie-algebroidal interpretation.
MacAdam makes this presentation precise: involution algebroids form a
tangent-categorical sketch, and their category in smooth manifolds is
isomorphic to the category of Lie algebroids
\citep{macadam2022functorial}.  For
each sketch object $x$, let $A_x\to M_x$ be a Lie algebroid over the realized
object $M_x=D(x)$, with anchor
\[
\rho_x:A_x\longrightarrow TM_x.
\]
The bundle $A_x$ represents the admissible infinitesimal learning operations,
whereas the anchor sends them to the actual tangent directions along which the
realized model can change.  A sketch arrow $f:x\to y$ is equipped, when
defined, with an algebroid transport $A_f:A_x\to A_y$ over $D(f)$ satisfying
the anchor square
\[
\begin{array}{ccc}
A_x & \xrightarrow{A_f} & A_y\\
\downarrow\rho_x && \downarrow\rho_y\\
TM_x & \xrightarrow{T(Df)} & TM_y.
\end{array}
\qquad
T(Df)\rho_x=\rho_yA_f.
\]
Thus the path equations of the learning sketch constrain not only the base
maps $D(f)$ but also their induced transports of admissible infinitesimal
operations.  For two parallel sketch paths, compositionality requires the
corresponding anchored transports to agree; tangent non-compositionality
measures failure of this compatibility after passage through the anchors.
The Lie-algebroid identity
\[
\rho_x([s,t]_{A_x})=[\rho_x(s),\rho_x(t)]
\]
then relates bracket closure of abstract learning operations to closure of
their realized vector fields.  This is an additional realization of a tangent
learning sketch, not a claim that every tangent sketch canonically carries a
Lie algebroid.  MacAdam's result provides the appropriate sketch semantics for
this additional structure; the anchor and bracket conditions below specify how
that structure acts on a particular learning diagram.  It connects \textsc{Lincs}   directly to LASKO, where skill edits are
modeled as anchored sections and order-sensitive skill interactions are
detected by their Lie brackets
\citep{macadam2022functorial,mahadevan2026lasko}.
\end{remark}

A failed factorization may indicate a hidden variable, a missing morphism, a defective representation, an unmodeled constraint, a failure of gluing, or a missing causal direction. The learning problem is not merely to reduce error, but to repair structure so that a universal property becomes closer to being satisfied.

\section{Examples of Infinitesimal Non-Compositionality}
\label{sec:inc-examples}

To make the \textsc{Lincs}   formalization more concrete, we now list several concrete instances of the \textsc{Lincs}   pattern. In each case, the base learning system has a compositionality defect, and the infinitesimal version asks whether that defect remains controlled after tangent lift. We spell out the first example in detail, since it is the running template for the later architectures.

\paragraph{Transformer equivariance.}\label{sec:transformer-equivariance}
The question posed in the introduction has a precise answer.  Let
$F_\theta:\mathbb R^{m\times d}\to\mathbb R^{m\times e}$ be a smooth
content-only Transformer and let a permutation matrix $P$ act on the token
axis.  Define the base equivariance defect
\[
E_{P,\theta}(X)=F_\theta(PX)-PF_\theta(X).
\]
This is the additive scalarizable representative of the obstruction to the
equivariance square.  Its first tangent lift, evaluated on an input
perturbation $V$, is
\[
D E_{P,\theta}(X)[V]
  =D F_\theta(PX)[PV]-P D F_\theta(X)[V].
\]
If parameters are also perturbed by $\dot\theta$, the full tangent defect adds
\[
\partial_\theta F_\theta(PX)[\dot\theta]
  -P\partial_\theta F_\theta(X)[\dot\theta].
\]
Consequently, tangent INC tests whether the local sensitivity of the
Transformer intertwines the permutation action, rather than merely whether
the two finite outputs happen to agree.

More generally, the $n$th input-level prolongation is represented by
\[
D^nF_\theta(PX)[PV_1,\ldots,PV_n]
  -P D^nF_\theta(X)[V_1,\ldots,V_n].
\]
Hence, if $F_\theta(PX)=PF_\theta(X)$ holds identically on a smooth invariant
domain, differentiating the identity shows that every displayed tangent and
higher-order obstruction vanishes.  The iterated \textsc{Lincs}   question is therefore
not mysterious in the ideal equivariant model: exact equivariance is inherited
by all tangent orders.  It becomes informative for learned or approximate
equivariance, restricted data support, finite-precision implementations, and
architectures containing positional encodings, causal masks, routing, or
other order-dependent components.  In those cases the base defect measures
output-level symmetry breaking, the first lift measures its sensitivity, and
higher lifts distinguish curvature and interaction effects that a small base
loss can conceal.  Thus the INC tower provides a graded diagnostic of where
and at what differential order a language model departs from its intended
equivariance sketch.

\paragraph{GIRL.}
Gradient Infinitesimal Reinforcement Learning (GIRL) is the \textsc{Lincs}   lift of gradient temporal-difference learning. Its base learning problem is policy evaluation in a Markov reward process: a value representation should be compatible with the Bellman update, and the tangent learning problem asks whether infinitesimal perturbations of states, values, parameters, or rewards preserve that compatibility.

Let \(\mathcal X\) denote the state object of an MRP, \(P^\pi:\mathcal X\to\mathcal X\) the transition morphism under policy \(\pi\), \(V_\theta:\mathcal X\to\R\) a parameterized value representation, and
\[
B_\pi:\R\to\R,
\qquad
B_\pi(v)=r+\gamma v,
\]
the scalar Bellman update induced by the reward and discount. The GIRL base square is
\[
\begin{array}{ccc}
\mathcal X & \xrightarrow{\ P^\pi\ } & \mathcal X\\
\downarrow V_\theta && \downarrow V_\theta\\
\R & \xrightarrow{\ B_\pi\ } & \R .
\end{array}
\]

As a learning sketch, GIRL is the quadruple
\[
\mathbb S_{\mathrm{GIRL}}
  =
  (S_{\mathrm{GIRL}},
   \mathcal D_{\mathrm{GIRL}},
   \mathcal L_{\mathrm{GIRL}},
   \mathcal K_{\mathrm{GIRL}}).
\]
The graph \(S_{\mathrm{GIRL}}\) has two formal objects \(x\) and \(v\), standing for states and values, and generating arrows
\[
p:x\to x,
\qquad
u:x\to v,
\qquad
b:v\to v.
\]
Under a model \(D_\theta:\Path(S_{\mathrm{GIRL}})\to\C\), these are interpreted as
\[
D_\theta(x)=\mathcal X,\quad
D_\theta(v)=\R,\quad
D_\theta(p)=P^\pi,\quad
D_\theta(u)=V_\theta,\quad
D_\theta(b)=B_\pi.
\]
The commutativity component is the single Bellman equation
\[
\mathcal D_{\mathrm{GIRL}}
  =
  \{\,u\circ p \sim b\circ u\,\}.
\]
Equivalently, the quotient \(q_{\mathcal D}:\Path(S_{\mathrm{GIRL}})\to
\Path(S_{\mathrm{GIRL}})/{\sim_{\mathcal D}}\) identifies the two paths
\[
x\xrightarrow{p}x\xrightarrow{u}v,
\qquad
x\xrightarrow{u}v\xrightarrow{b}v.
\]
The limiting component \(\mathcal L_{\mathrm{GIRL}}\) records the admissible feature/value representation structure, for example the finite-dimensional parameter or feature cone through which \(V_\theta\) is required to factor in linear TD. The colimiting component \(\mathcal K_{\mathrm{GIRL}}\) records the sampling or empirical aggregation structure used to form the TD estimating equations, such as the finite-sample cocone that aggregates transition observations into the stochastic approximation objective. In the minimal Bellman sketch these two classes may be empty; in GTD and GTD-MP they carry the extra representation and empirical averaging structure needed by the saddle objective.

The base non-compositionality is therefore not the arithmetic Bellman residual itself. It is the obstruction
\[
\Obs\big(\Fact_{\mathbb S_{\mathrm{GIRL}}}(D_\theta)\big)
\]
to realizing the Bellman sketch. A scalar TD, GTD, or GTD-MP loss is obtained only after applying a scalarization to this obstruction.

The corresponding tangent learning sketch equips \(\mathbb S_{\mathrm{GIRL}}\) with the Cockett-Cruttwell tangent action. Its model sends \(D_\theta\) to
\[
\T D_\theta:\Path(S_{\mathrm{GIRL}})\to\C,
\]
whose interpreted arrows are \(\T P^\pi\), \(\T V_\theta\), and \(\T B_\pi\). The tangent Bellman square is
\[
\begin{array}{ccc}
\T\mathcal X & \xrightarrow{\ \T P^\pi\ } & \T\mathcal X\\
\downarrow \T V_\theta && \downarrow \T V_\theta\\
\T\R & \xrightarrow{\ \T B_\pi\ } & \T\R .
\end{array}
\]
The GIRL infinitesimal non-compositionality is
\[
\INC_{\mathrm{GIRL}}(D_\theta)
  =
  \Obs\big(\Fact_{\mathbb S_{\mathrm{GIRL}}}(\T D_\theta)\big).
\]
After scalarization, this yields the tangent Bellman error developed in the
companion GIRL paper \citep{mahadevan2026girl}.  At the categorical level considered here, GIRL augments
the base Bellman factorization problem with an infinitesimal Bellman-closure
term; the concrete LINCS-GTD-MP saddle formulation is introduced in that
companion work.

\begin{center}
\begin{tabular}{ll}
\toprule
Problem class & Representative obstruction or loss\\
\midrule
GIRL~\citep{mahadevan2026girl} & Bellman and tangent Bellman loss\\
DB~\citep{mahadevan2026catagi} & compositional backpropagation and tangent diagram loss\\
BRIDGE/SKFM/IC~\citep{mahadevan2026bridge} & causal Lie-bracket closure loss\\
ALLORA~\citep{mahadevan2026allora} & LoRA matrix commutator and quotient interaction profiles\\
LASKO~\citep{mahadevan2026lasko} & Markdown-category skill-composition loss\\
\bottomrule
\end{tabular}
\end{center}

A detailed experimental evaluation of \textsc{Lincs} across these problem
classes---GIRL, DB, BRIDGE/SKFM/IC, ALLORA, and LASKO---is currently underway
and will be reported in future work.

\paragraph{GIKET.}
Kan Extension Transformers learn by extending local neighborhoods into global representations \citep{mahadevan2026ket}. The base defect is a failure of neighborhood extension or Kan invariance; the INC lift asks whether the extension varies coherently under infinitesimal perturbations of tokens, embeddings, or neighborhoods.

\paragraph{Differential DB.}
Diagrammatic Backpropagation repairs noncommuting computational diagrams, and is the canonical example from which \textsc{Lincs}   grows. A useful concrete instance is Sudoku solving. A candidate Sudoku grid determines many local views: its rows, columns, blocks, and individual cells. A correct solution is not merely a vector of predicted digits; it is a diagram whose overlapping local views agree and whose rows, columns, and blocks satisfy the Sudoku constraints.

Let \(X\) be a formal object for a candidate grid, \(R_i\) the \(i\)-th row object, \(K_j\) the \(j\)-th column object, \(B_b\) the \(b\)-th block object, and \(C_{ij}\) the cell object at position \((i,j)\). The Sudoku DB learning sketch is
\[
\mathbb S_{\mathrm{DB}}
  =
  (S_{\mathrm{DB}},
   \mathcal D_{\mathrm{DB}},
   \mathcal L_{\mathrm{DB}},
   \mathcal K_{\mathrm{DB}}).
\]
The graph \(S_{\mathrm{DB}}\) contains the view maps
\[
\rho_i:X\to R_i,\qquad
\kappa_j:X\to K_j,\qquad
\beta_b:X\to B_b,
\]
and the cell-projection maps
\[
r_{ij}:R_i\to C_{ij},\qquad
k_{ij}:K_j\to C_{ij},\qquad
b_{ij}:B_b\to C_{ij},
\]
where \(b\) is the block containing cell \((i,j)\). The commutativity component \(\mathcal D_{\mathrm{DB}}\) says that every overlapping view gives the same cell:
\[
r_{ij}\rho_i
\sim
k_{ij}\kappa_j
\sim
b_{ij}\beta_b
\qquad
\text{for all }(i,j).
\]
The limiting component \(\mathcal L_{\mathrm{DB}}\) records that a grid is assembled from compatible local cell assignments, for example as a product-like cone of cell predictions. The colimiting component \(\mathcal K_{\mathrm{DB}}\) records gluing: row, column, and block views are identified along their shared cells. Additional Sudoku validity maps can be included in \(S_{\mathrm{DB}}\), such as maps \(R_i,K_j,B_b\to \mathsf{Perm}_9\), whose sketch constraints require each row, column, and block to realize the digit set exactly once.

A neural Sudoku solver is then a candidate model
\[
D_\theta:\Path(S_{\mathrm{DB}})\to\C
\]
whose objects may be soft digit distributions, logits, constraint states, or learned local representations. The DB obstruction
\[
\Obs\big(\Fact_{\mathbb S_{\mathrm{DB}}}(D_\theta)\big)
\]
measures failure of the predicted local views to glue into a globally consistent Sudoku solution. A scalar DB loss is obtained by measuring these overlap and constraint obstructions.

The tangent learning sketch for Differential DB applies the tangent structure to the same Sudoku factorization problem:
\[
\T D_\theta:\Path(S_{\mathrm{DB}})\to\C.
\]
Its INC term is
\[
\INC_{\mathrm{DB}}(D_\theta)
  =
  \Obs\big(\Fact_{\mathbb S_{\mathrm{DB}}}(\T D_\theta)\big).
\]
This asks whether infinitesimal changes in logits, givens, local cell beliefs, or constraint messages preserve row-column-block compatibility. In a crossword puzzle, the analogous sketch replaces Sudoku rows, columns, and blocks by across clues, down clues, and shared letter cells. The Berkeley Crossword Solver uses neural candidate generation followed by loopy belief propagation and local search to enforce such overlap constraints among potential answers \citep{wallace2022automated}. In \textsc{Lincs}   terms, DB measures failure of clue-wise word assignments to agree on overlaps, while Differential DB measures whether that agreement is stable under infinitesimal changes in clue embeddings or letter beliefs.

\paragraph{ALLORA.}
ALLORA trains low-rank neural adapters with Lie-algebraic commutator penalties so independently useful adapters compose more predictably \citep{mahadevan2026allora}. Its base defect is adapter order-sensitivity; its reflective \textsc{Lincs} version studies quotient-projectable adapter-flow interactions, including both bracket obstructions and connection-dependent symmetric acceleration in the representation bundle.

The low-rank presentation $\Delta=BA$ has a $\mathrm{GL}(r)$ redundancy,
$(B,A)\mapsto(BQ,Q^{-1}A)$.  Consequently, an optimizer field in factor
coordinates is not automatically a vector field on the quotient of effective
updates.  A valid tangent realization must either fix a balanced gauge or
provide an equivariant horizontal lift through the presentation map
$\pi(B,A)=BA$.  Once this descent datum is fixed, two task/commutator
directional interactions may be retained separately.  The Lie bracket records
their antisymmetric difference; a declared connection also defines their
symmetric acceleration.  Which component is useful for predicting adapter
composition is a sketch- and realization-specific question, not a categorical
preference for antisymmetry.

\paragraph{LASKO.}
LASKO models agentic skill optimization over controlled Lie algebroids, where edits to prompts, schemas, tools, validators, and workflow artifacts are treated as anchored sections whose brackets reveal order-sensitive repair interactions \citep{mahadevan2026lasko}. Its \textsc{Lincs}   version studies vector-bundle bracket defects for latent knowledge and skill-update manifolds.

\section{Diagrammatic Backpropagation Revisited}

Diagrammatic Backpropagation can be expressed as the repair of factorization failures \citep{mahadevan2026catagi}. Let \(D_\theta:J\to\C\) be a learned diagram depending on parameters \(\theta\), obtained from the path category of a learning sketch \(\mathbb S=(S,\mathcal D,\mathcal L,\mathcal K)\), and let \(q=q_{\mathcal D}:J\to J/{\sim_{\mathcal D}}\) encode the intended commutativity constraints. DB attempts to adjust \(\theta\) so that \(D_\theta\) factors through \(q\), possibly together with the limit and colimit constraints specified by \(\mathcal L\) and \(\mathcal K\).

In enriched settings, this factorization problem is often scalarized by a loss:
\[
\Loss_{DB}(\theta)=\Phi(\Fact_q(D_\theta)),
\]
where \(\Phi\) is an application-specific measure of factorization failure, such as a norm, energy, likelihood, divergence, Bregman distance, or empirical risk. Such scalarizations are important computationally, but they are not primitive in the categorical definition.

DB therefore replaces pointwise error with compositional error. The learning signal is no longer only ``prediction minus target''; it is a failure of a structured diagram to satisfy a universal compositionality condition.

\begin{remark}
Ordinary losses fit this pattern after scalarization. Supervised learning uses a diagram in which an input is mapped by a model and compared with a target map. Autoencoding compares an input with the reconstruction composite. Bellman residual minimization compares a value function with its Bellman transform. Contrastive learning compares two representation paths induced by augmentations. Sheaf learning compares restrictions along overlaps.
\end{remark}

\subsection{From global to local non-compositionality}

Diagrammatic Backpropagation can also be read as a localization principle.  A
global learning diagram is covered by smaller computational diagrams, and its
global failure of compositionality is restricted to failures attached to those
local pieces.  This suggests a categorical generalization of the role played
by the chain rule in ordinary backpropagation.

Let \(\{i_a:J_a\to J\}_{a\in A}\) be a family of subsketch inclusions whose
images cover the objects, generating arrows, and specified constraints of a
learning sketch \(\mathbb S\).  Write \(\mathbb S_a\) for the induced local
sketch and \(D_a=D i_a\) for the restriction of a model \(D:J\to\C\).

\begin{definition}[Obstruction localization]
An obstruction localization for the cover \(\{J_a\}_{a\in A}\) is a natural
assignment
\[
\Lambda_D:
\Obs\big(\Fact_{\mathbb S}(D)\big)
\longrightarrow
\mathsf{Comp}\!\left(
  \{\Obs(\Fact_{\mathbb S_a}(D_a))\}_{a\in A}
\right),
\]
where \(\mathsf{Comp}\) denotes families of local obstruction data satisfying
the required compatibility conditions on overlaps \(J_a\times_JJ_b\).  The
localization is \emph{complete} when a compatible family of inhabited local
factorization problems glues to an inhabited global factorization problem.
\end{definition}

The compatibility term is essential.  Vanishing of each isolated local defect
need not imply global compositionality unless the chosen local repairs agree on
their shared objects, paths, cones, and cocones.  Thus obstruction localization
is a descent statement for factorization problems, rather than an additive
decomposition of a numerical error.

\begin{proposition}[Tangent stability of obstruction localization]
Suppose restriction along each \(i_a\) preserves admissible models and
factorization problems, the obstruction assignment is natural under these
restrictions, and tangent lift commutes with restriction:
\[
(\T D)i_a\cong\T(Di_a).
\]
Then every obstruction localization \(\Lambda_D\) has a tangent localization
\[
\Lambda_{\T D}:
\INC(D)
\longrightarrow
\mathsf{Comp}\!\left(
  \{\INC(D_a)\}_{a\in A}
\right).
\]
If the base localization is complete and tangent lift preserves the overlap
descent data, the tangent localization is complete as well.
\end{proposition}

\begin{proof}
Restricting the tangent model to \(J_a\) gives
\((\T D)i_a\cong\T D_a\).  Naturality of the factorization and obstruction
assignments therefore sends the global tangent obstruction
\(\Obs(\Fact_{\mathbb S}(\T D))\) to the compatible family
\(\{\Obs(\Fact_{\mathbb S_a}(\T D_a))\}_{a\in A}\), which is precisely the
displayed family of local INC objects.  Preservation of overlap descent
transports the local-to-global gluing property to the tangent level.\qedhere
\end{proof}

For a feedforward network, the indexing category is a chain of layer maps and
the terminal supervised error is a scalarization of a global factorization
obstruction.  Reverse-mode differentiation applies the chain rule to transport
the sensitivity of that scalarization through the chain, producing a local
parameter-update signal at each layer.  In this sense, ordinary
backpropagation is the scalarized, additive instance of obstruction
localization.  \textsc{Lincs}   retains the same global-to-local architecture before
scalarization: it localizes base obstructions, tangent obstructions, and, when
defined, interaction-profile or higher-order obstructions.  This opens the possibility of
``backpropagating'' structured obstruction objects even when subtraction,
norms, or scalar losses are unavailable.

\section{Tangent Lift of Diagrammatic Backpropagation}

We now assume that \(\C\) carries tangent structure. Tangent categories axiomatize the behavior of tangent bundles in a categorical setting \citep{cockett2014tangent}. A tangent category has an endofunctor
\[
\T:\C\to\C
\]
together with structure maps such as projection, zero, addition, vertical lift, and canonical flip, satisfying axioms abstracting the tangent bundle of a smooth manifold.

In this paper, tangent learning sketches are the bridge between these two levels: they specify ordinary learning constraints and require those constraints to remain meaningful after applying the tangent functor. Thus the tangent functor lifts not only objects and morphisms in \(\C\), but also the sketch-theoretic factorization problems that define compositional learning.

The tangent-category axioms also support a notion of vector field and, in the standard tangent-categorical setting, a Lie bracket of vector fields. This is important for \textsc{Lincs}   because infinitesimal coherence is not exhausted by the existence of \(\Fact_{\mathbb S}(\T D)\). A learning diagram may generate tangent directions that are individually admissible, while their bracket lies outside the proper admissible distribution represented by the current architecture or repair language. In that case the learned infinitesimal dynamics are not closed relative to that distribution. A connection, when additionally declared, exposes symmetric as well as antisymmetric second-order interaction.

\begin{definition}[Tangent model]
Let \(D:J\to\C\) be a model of a learning sketch \(\mathbb S=(S,\mathcal D,\mathcal L,\mathcal K)\), with \(J=\Path(S)\). Its tangent model is
\[
\T D:J\to\C,
\]
obtained by applying the tangent functor objectwise and morphismwise:
\[
(\T D)(j)=\T(Dj),
\qquad
(\T D)(u)=\T(Du).
\]
\end{definition}

\begin{remark}[Weil-algebra indexing of tangent models]
Leung's classifier makes the iteration in this definition explicit
\citep[Sec.~14]{leung2017classifying}.  For the dual-number object
$W=\mathbb N[x]/(x^2)$ and the classifying functor
$F:\mathbb N\text{-}\mathsf{Weil}_1\to\operatorname{End}(\C)$, one has
$T=F(W)$.  Strong monoidality identifies the $n$-fold iterate with
\[
T^n\cong F(nW),
\]
where $nW$ is the $n$-fold coproduct of $W$, whereas the fiber-product power
$T^{(n)}$ used in the tangent axioms is represented by $F(W^n)$.  Consequently
$\operatorname{INC}^{(n)}(D)$ may be read as the obstruction to the learning
factorization problem after the Weil prolongation indexed by $nW$.  The
distinction between $nW$ and $W^n$ separates iterated or mixed infinitesimal
directions from the pullback powers required for additive tangent bundles.
\end{remark}

The base factorization problem asks whether computations compose:
\[
\Fact_{\mathbb S}(D)\neq\varnothing.
\]
The tangent factorization problem asks whether infinitesimal perturbations compose:
\[
\Fact_{\mathbb S}(\T D)\neq\varnothing.
\]

\begin{remark}[Tangent structure and quotient sketches]
The Cockett-Cruttwell axioms require the tangent functor to preserve the finite limits used in the tangent structure, together with compatibility of the projection, zero, addition, vertical lift, and canonical flip. The quotient \(q_{\mathcal D}:J\to J/{\sim_{\mathcal D}}\), however, is induced by the commutativity component of a sketch and may involve colimit-like data in the category of indexing shapes. The cone and cocone components \(\mathcal L,\mathcal K\) add further universal constraints. \textsc{Lincs}   therefore does not assume that \(\T\) literally preserves the quotient construction or all sketch data as colimits. Instead, it assumes that the \emph{factorization problem induced by the sketch} is stable under tangent lift: whenever \(D:J\to\C\) is admissible, \(\T D\) is again an admissible model of the same learning sketch, and \(\Fact_{\mathbb S}(\T D)\) is the transported tangent factorization problem.
\end{remark}

\begin{condition}[Tangent-stable factorization fibration]
The factorization fibration \(p:\mathsf{Fact}(\C)\to\mathsf{Learn}(\C)\) is tangent-stable when there is a lift of the tangent functor to factorization problems,
\[
\begin{array}{ccc}
\mathsf{Fact}(\C) & \xrightarrow{\T_{\mathsf{Fact}}} & \mathsf{Fact}(\C)\\
\downarrow p && \downarrow p\\
\mathsf{Learn}(\C) & \xrightarrow{\T} & \mathsf{Learn}(\C),
\end{array}
\]
so that the fiber over \(D\) is sent to the fiber over \(\T D\). This is the fibrational compatibility condition replacing any blanket assumption that \(\T\) preserves all quotient sketches as colimits.
\end{condition}

\begin{definition}[Scalarized \textsc{Lincs}   objective]
When scalarizations are available, a \textsc{Lincs}   objective associated with a diagram \(D_\theta\) is
\[
\Loss_{LINCS}(\theta)
  =
  \Loss_{base}\big(\Obs(\Fact_q(D_\theta))\big)
  +
  \lambda\Loss_{tan}\big(\INC(D_\theta)\big),
\]
where $\lambda=0$ is permitted. More generally, scalarized \textsc{Lincs}
may include higher-order tangent, bracket, connection, curvature, or jet terms.
These are candidate signals supplied by the categorical realization; their
nonzero computational admission is an additional validation decision rather
than an axiom of tangent compositionality.
\end{definition}

Thus \textsc{Lincs}   is not a new optimizer. It is a rule for lifting learning compositionality problems into tangent structure.

\section{\textsc{Lincs}   Axioms}

In this section, we propose an axiomatic structure for \textsc{Lincs}  . First, we state the underlying theoretical object more formally. 

\begin{definition}[Tangent compositionality problem]
Let \(D:J\to\C\) be a model of a learning sketch \(\mathbb S=(S,\mathcal D,\mathcal L,\mathcal K)\), with commutativity quotient \(q=q_{\mathcal D}:J\to J/{\sim_{\mathcal D}}\). The tangent compositionality problem of \(D\) is
\[
\Fact_{\mathbb S}(\T D),
\]
or \(\Fact_q(\T D)\) when only the commutativity quotient is under discussion. It asks whether the tangent-lifted model \(\T D\) satisfies the same sketch-theoretic factorization problem as \(D\).
\end{definition}

\begin{definition}[Infinitesimal Non-Compositionality]
The infinitesimal non-compositionality of \(D\) is the obstruction to inhabiting its tangent compositionality problem:
\[
\INC(D)
  :=
  \Obs\big(\Fact_{\mathbb S}(\T D)\big).
\]
When higher-order tangent structure is available, the \(n\)-th order INC is
\[
\INC^{(n)}(D)
  :=
  \Obs\big(\Fact_{\mathbb S}(\T^n D)\big).
\]
\end{definition}

Thus INC has a factorization form.  A learning diagram may additionally
determine vector fields, tangent sections, a connection, or other prolonged
operations.  These data refine tangent INC, but they do not replace its
factorization definition.

\begin{definition}[Interaction signature]
An interaction signature $\Omega_D$ for an admissible model $D$ is a declared
family of operations on an admissible tangent distribution
$\mathcal A_D\subseteq\Vect_D$.  It may contain the Lie bracket, covariant
directional derivatives for a specified connection, curvature or torsion,
canonical-flip comparisons, or higher jet operations.  The signature is part
of the learning sketch or its realization; no single second-order operation
is required for every sketch.
\end{definition}

\begin{definition}[Connection-aware second-order profile]
Suppose the realization of $D$ carries a connection $\nabla$ and generates
admissible fields $X,Y$.  Its ordered second-order profile is
\[
J^2_{\nabla,D}(X,Y)
  =
  \big(\nabla_XY,\nabla_YX\big).
\]
The associated antisymmetric and symmetric components are
\[
A_\nabla(X,Y)
  =\frac12\big(\nabla_XY-\nabla_YX\big),
\qquad
S_\nabla(X,Y)
  =\frac12\big(\nabla_XY+\nabla_YX\big).
\]
For a torsion-free connection, $2A_\nabla(X,Y)=[X,Y]$.  In contrast,
$S_\nabla$ depends on the declared connection and is not primitive tangent
category structure.
\end{definition}

\begin{definition}[Bracket-closure INC]
Let \(\Vect_D\) denote the class of tangent directions, vector fields, or tangent sections generated by an admissible learning model \(D\), and let
$\mathcal A_D\subseteq\Vect_D$ be the directions declared admissible for the
current architecture or repair language. The bracket-closure INC of \(D\) is the closure problem
\[
\Cl_{[,]}(\mathcal A_D):
\qquad
  \{\, [X,Y]\mid X,Y\in\mathcal A_D\,\}
  \subseteq
  \mathcal A_D.
\]
The model is bracket-closed relative to $\mathcal A_D$ when this inclusion
holds. Failure of the inclusion is a second-order infinitesimal
non-compositionality.  Taking $\mathcal A_D$ to be the entire tangent bundle
would make this test vacuous for genuine smooth vector fields, whose bracket is
automatically tangent; structural nonclosure is meaningful relative to a
proper admissible distribution such as a fixed adapter placement, shared
factor subspace, or repair vocabulary.
\end{definition}

This definition avoids subtracting morphisms in the base category. The primitive obstruction is not a residual \(g\circ f-f\circ g\). It is the failure of tangent directions generated by non-compositional learning data to be involutive under the Lie bracket supplied by the tangent structure.  Bracket nonclosure is one interaction obstruction.  A connection-aware realization may instead or additionally observe $S_\nabla$, $A_\nabla$, or the full ordered profile $J^2_\nabla$ without asserting that antisymmetry is the unique carrier of useful tangent information.

If \(\Fact_q(D)\) is scalarized by a loss, then \(\Obs(\Fact_q(\T D))\) may induce a gradient-like or Jacobian-like regularizer. But \textsc{Lincs}   is not merely gradient regularization. The object being lifted is not necessarily a scalar objective; it is the universal compositionality problem whose obstruction produced the loss.

\begin{remark}
Backpropagation differentiates a scalar loss. \textsc{Lincs}   lifts the categorical factorization problem whose scalarization produced the loss. This distinction matters because the tangent lift can preserve structure that scalarization discards.
\end{remark}

\subsection{Axioms}

We now propose an axiomatization of \textsc{Lincs}   categories.

\begin{definition}[\textsc{Lincs}   object]
A \textsc{Lincs}   object is a pair
\[
(X,\mathcal S_X),
\]
where \(X\in\Ob(\C)\) and \(\mathcal S_X\) is a specified class of learning sketches whose models land in, are indexed by, or are interpreted around \(X\). Informally, \(X\) is a state, representation, policy, local model, or world-model object together with the compositionality problems it is expected to satisfy.
\end{definition}

\begin{definition}[\textsc{Lincs}   morphism]
A \textsc{Lincs}   morphism
\[
f:(X,\mathcal S_X)\to(Y,\mathcal S_Y)
\]
is a morphism \(f:X\to Y\) in \(\C\) together with a transport of learning sketches and their models from \(\mathcal S_X\) to \(\mathcal S_Y\), such that tangent lift commutes with this transport. Concretely, if \(D:J\to\C\) is a transported model, then transporting \(\T D\) agrees with tangent-lifting the transported model.
\end{definition}

\begin{definition}[INC category]
An INC category is a tangent category \((\C,\T)\) equipped with:
\begin{enumerate}[leftmargin=2em]
  \item a class \(\mathsf{Sk}_{\C}\) of tangent learning sketches \(\mathbb S=(S,\mathcal D,\mathcal L,\mathcal K)\);
  \item for each \(\mathbb S\), with \(J=\Path(S)\), a class of admissible models \(D:J\to\C\);
  \item for each sketch \(\mathbb S\), the associated base factorization problem \(\Fact_{\mathbb S}(D)\), including the quotient induced by \(\mathcal D\) and the cone/cocone constraints induced by \(\mathcal L,\mathcal K\);
  \item for each admissible model \(D\), the tangent model \(\T D:J\to\C\) and tangent factorization problem \(\Fact_{\mathbb S}(\T D)\);
  \item when higher interaction data are assigned to learning models, an
  interaction signature $\Omega_D$, an admissible distribution $\mathcal A_D$,
  and the corresponding closure or profile problem; bracket closure
  $\Cl_{[,]}(\mathcal A_D)$ is the intrinsic antisymmetric specialization.
\end{enumerate}
\end{definition}

\begin{axiom}[Base compositionality]
For every admissible model \(D:J\to\C\) of a learning sketch \(\mathbb S\), \(D\) is compositional if and only if \(\Fact_{\mathbb S}(D)\) is inhabited.
\end{axiom}

\begin{axiom}[Tangent admissibility]
If \(D:J\to\C\) is an admissible model of a tangent learning sketch \(\mathbb S\), then \(\T D:J\to\C\) is also an admissible model of \(\mathbb S\), and the assignment \(D\mapsto \T D\) agrees with the tangent action specified by the sketch.
\end{axiom}

\begin{axiom}[Presentation descent]
Suppose a computational realization is expressed in a parameter presentation
$\pi:P\to M$ with redundant representatives.  An optimizer or repair field
$V_P:P\to TP$ is admissible as a field on the model object $M$ only when it is
$\pi$-projectable: there is a field $V_M:M\to TM$ such that
\[
T\pi\circ V_P=V_M\circ\pi.
\]
Alternatively, the realization must declare a gauge fixing or an equivariant
horizontal lift $h$ satisfying $T\pi\circ h(V_M)=V_M$.  Tangent obstruction
signals that fail this descent condition are presentation artifacts and are
not INC data on $M$.
\end{axiom}

\begin{axiom}[Tangent compositionality]
For every admissible model \(D:J\to\C\) of a learning sketch \(\mathbb S\), infinitesimal compositionality is the factorization property
\[
\Fact_{\mathbb S}(\T D)\neq\varnothing.
\]
Thus tangent compositionality is defined by the same universal property as base compositionality, but applied after the tangent functor. INC is the obstruction \(\Obs(\Fact_{\mathbb S}(\T D))\) to this property.
\end{axiom}

\begin{axiom}[Sketch-specific interaction]
For every admissible model $D$ equipped with an interaction signature
$\Omega_D$ and admissible tangent distribution $\mathcal A_D$, each operation
declared by $\Omega_D$ determines a closure, factorization, or profile problem.
Failure of a declared closure or factorization property is an admissible INC
signal.  When $[,]\in\Omega_D$, this includes the relative bracket-closure
problem
\[
\Cl_{[,]}(\mathcal A_D):
\qquad
[X,Y]\in\mathcal A_D
\quad
\text{for all }X,Y\in\mathcal A_D.
\]
When a connection belongs to the realization, $J^2_\nabla$, $S_\nabla$, and
$A_\nabla$ are admissible profile data.  The sketch does not categorically
mandate that the antisymmetric projection be primary.
\end{axiom}

\begin{axiom}[Conditional non-compositionality transport]
For a model and sketch carrying an interaction signature, a declared transport
is a natural assignment
\[
\mathfrak t_{D,\mathbb S}:
\Obs\big(\Fact_{\mathbb S}(D)\big)
\rightsquigarrow
\Obs_{\Omega,\mathbb S}(\T D).
\]
The dashed arrow emphasizes that transport is additional sketch-relative data,
not an equality between base failure and bracket closure.  An admissible
transport must be natural under model transformations, descend through any
parameter presentation, and preserve the semantics of the factorization
problem being repaired.  Bracket transport is recovered when
$\Omega_D=\{[,]\}$; its existence alone does not imply that its scalarization
is informative for a particular repair target.
\end{axiom}

\begin{axiom}[Functoriality of repair]
If a transformation of models \(\alpha:D\to D'\) is an admissible repair, then \(\T\alpha:\T D\to\T D'\) is an admissible infinitesimal repair. Moreover, repair preserves factorization: if \(D'\) factors through \(q\), then repaired models factor through \(q\) whenever the repair is declared compositionality-preserving.
For a cover by local subsketches, a repair declared descent-compatible must
also commute with restriction up to the specified coherence,
\[
\res_U\circ R_G
\;\cong\;
R_U\circ\res_U,
\]
and satisfy the corresponding cocycle conditions on overlaps.  Functorial
repair, descent-compatible repair, and empirically useful repair are distinct
notions.
\end{axiom}

\begin{axiom}[Optional scalarization]
An INC category may be equipped with scalarization maps
\[
\Phi_q:\{\text{obstructions to factorization problems over }q\}\to\R_{\ge 0},
\]
but such maps are additional computational structure, not part of the bare categorical definition. When present, \(\Phi_q\) should vanish on inhabited factorization problems:
\[
\Fact_q(D)\neq\varnothing
\quad\Rightarrow\quad
\Phi_q\big(\Obs(\Fact_q(D))\big)=0.
\]
\end{axiom}

\begin{condition}[Observation specificity and identifiability]
A computational realization may retain a profile-valued observation
$\Psi_{\mathbb S}(D,e)$ before scalarization, containing local, overlap,
counterfactual, or downstream effects of a candidate repair $e$.  Structural
identification from this profile requires the separate condition
\[
\Psi_{\mathbb S}(D,e)=\Psi_{\mathbb S}(D,e')
\quad\Longrightarrow\quad
e\simeq e'.
\]
When this implication fails, the sketch identifies only an equivalence class
or moduli space of repairs.  Defect reduction alone does not imply recovery of
a privileged architecture.
\end{condition}

\begin{definition}[Validated scalarized realization]
A validated scalarized \textsc{Lincs} realization consists of profile or
obstruction scalarizations together with an admission rule, fitted without
test leakage, that may assign zero weight to any auxiliary tangent signal.  A
typical base-first objective is
\[
\Loss_D^{\mathrm{val}}
  =
  \Phi_q\big(\Obs(\Fact_q(D))\big)
  +
  \sum_{k\in\Omega_D}g_k(\mathcal V)\lambda_k\Phi_k(\INC_k(D)),
\]
where $\mathcal V$ is validation data, $g_k(\mathcal V)\in\{0,1\}$ or
$\mathbb R_{\geq0}$, and sparse solutions with $g_k=0$ are admissible.  For a
connection-aware two-field realization the learned coefficients may be
written $(\mu_s,\mu_a)$ for symmetric and antisymmetric profiles.  If these
profiles are inserted directly into a vector update rather than a scalar
loss, signed regularized coefficients and the appropriate step-size scaling
are allowed.
\end{definition}

\begin{condition}[Stochastic consistency]
Suppose an interaction profile is estimated from minibatches by
$\widehat\Psi_K$.  A stochastic \textsc{Lincs} realization is consistent when
its aggregated estimator converges to the profile of the expected admissible
field,
\[
\widehat\Psi_K
\longrightarrow
\Psi_{\mathbb E}
\]
in the declared probabilistic or metric sense.  A finite-sample certificate
may take the form
\[
\mathbb E\norm{\widehat\Psi_K-\Psi_{\mathbb E}}^2
\leq \frac{C}{K}+b_K^2,
\qquad b_K\to0.
\]
Same-minibatch and independently crossed directional terms may have different
variance because their covariance differs, even when they converge to the
same expected-field profile.  This estimator convergence is distinct from the
contractivity of the coalgebraic INC endofunctor.
\end{condition}

\begin{definition}[\textsc{Lincs}   category]
A \textsc{Lincs}   category is an INC category together with a rule that, for every admissible model \(D:J\to\C\), regards the pair
\[
\big(\Fact_q(D),\Fact_q(\T D)\big)
\]
as the basic learning datum. When an interaction signature is present, the datum is refined to
\[
\big(\Fact_q(D),\Fact_q(\T D),\Obs_{\Omega,D}(\T D)\big),
\]
with $\Cl_{[,]}(\mathcal A_D)$ one possible component. If optional scalarizations are present, this datum may be converted into an objective
\[
\Loss_D
  =
  \Phi_q\big(\Obs(\Fact_q(D))\big)
  +
  \lambda\Phi_q^T\big(\INC(D)\big)
  +
  \sum_{k\in\Omega_D}\mu_k\Phi_q^k\big(\INC_k(D)\big).
\]
The bare category supplies candidate interaction data; a validated scalarized
realization determines which coefficients $\mu_k$ are nonzero.
\end{definition}

\begin{proposition}[Functoriality of INC data]
Let \(F:\C\to\D\) be a tangent functor between tangent categories, so that \(F\T_{\C}\cong \T_{\D}F\). Suppose \(F\) preserves learning sketches, admissible models, and quotient factorization problems. Then \(F\) maps INC data in \(\C\) to INC data in \(\D\):
\[
F\big(\Obs(\Fact_q(\T_{\C}D))\big)
  \longmapsto
\Obs\big(\Fact_{Fq}(\T_{\D}(FD))\big).
\]
In particular, tangent functors preserving the learning structure preserve tangent compositionality problems and their obstructions.
If interaction signatures, connections, admissible distributions, or
presentation quotients are included, preservation of those additional data is
a separate hypothesis; it does not follow from tangent functoriality alone.
\end{proposition}

\begin{proof}
Since \(F\) preserves the learning sketch and quotient, the base factorization problem \(\Fact_q(D)\) is sent to the corresponding factorization problem \(\Fact_{Fq}(FD)\). Since \(F\) is tangent, \(F(\T_{\C}D)\cong \T_{\D}(FD)\). Therefore the image of the tangent factorization problem \(\Fact_q(\T_{\C}D)\) is the tangent factorization problem \(\Fact_{Fq}(\T_{\D}(FD))\). Applying the obstruction assignment gives the claimed map on INC data.
\end{proof}

\subsection{Universal Properties}

The universal property question is: what is the minimal tangent learning structure generated by a category of base learning diagrams?

\begin{definition}[\textsc{Lincs}   completion, informal]
Let \(\C\) be a category equipped with learning sketches and admissible models. A \textsc{Lincs}   completion of \(\C\) is a tangent category \(\LINCS(\C)\) together with an embedding
\[
i:\C\to\LINCS(\C)
\]
such that base factorization problems in \(\C\) acquire tangent factorization problems in \(\LINCS(\C)\), and any sketch/model-preserving functor from \(\C\) into a \textsc{Lincs}   category factors through \(i\).
\end{definition}

\begin{conjecture}[Free \textsc{Lincs}   completion]
For a suitable class of sketch-equipped categories \(\C\), there exists a free \textsc{Lincs}   category \(\LINCS(\C)\) satisfying the universal property:
\[
\Hom_{\mathsf{LINCSCat}}(\LINCS(\C),\D)
  \cong
  \Hom_{\mathsf{SkCat}}(\C,U\D),
\]
where \(U\) forgets tangent factorization structure.
\end{conjecture}

This conjecture is the \textsc{Lincs}   analogue of free tangent or differential completions. It says that \textsc{Lincs}   is not merely an added regularizer; it is a universal completion that freely adds infinitesimal factorization structure to a base learning category.  For interaction-enriched \textsc{Lincs}, the universal property is relative to the declared signature $\Omega$: a bracket-closed completion, a connection-equipped jet completion, and a higher-operation completion are distinct specializations rather than one mandatory Lie-algebraic completion.

\section{\textsc{Lincs}   as a Coalgebra}

There is a second universal perspective on \textsc{Lincs}  , related to the standard coalgebraic treatment of state-based systems and final semantics \citep{rutten2000universal}. A tangent category contains an endofunctor
\[
\T:\C\to\C.
\]
An endofunctor does not by itself make every object into a coalgebra, but it determines a category of coalgebras once one chooses structure maps
\[
\gamma:X\to\T X.
\]
In a tangent category, vector fields are the basic examples: a vector field on \(X\) is a section \(v:X\to\T X\) of the tangent projection \(p_X:\T X\to X\). Thus vector fields can be read as \(T\)-coalgebras compatible with tangent structure.

This gives a coalgebraic reading of INC. A learning model \(D:J\to\C\) has a tangent tower
\[
D,\quad \T D,\quad \T^2D,\quad \ldots
\]
and a corresponding tower of factorization problems
\[
\Fact_q(D),\quad
\Fact_q(\T D),\quad
\Fact_q(\T^2D),\quad \ldots .
\]
The sequence resembles a system repeatedly unfolding its own infinitesimal behavior. Each stage asks whether the previous compositionality problem remains coherent after one more application of the tangent functor.

\begin{definition}[INC endofunctor]
The INC endofunctor is the operation
\[
\T_{\INC}:\mathsf{INC}(\C)\to\mathsf{INC}(\C)
\]
defined on admissible tangent factorization problems by
\[
\T_{\INC}\big(\Fact_q(\T^nD)\big)
  =
  \Fact_q(\T^{n+1}D),
\]
whenever these iterated tangent models are admissible.
\end{definition}

Equivalently, after applying the obstruction assignment,
\[
\T_{\INC}\big(\Obs(\Fact_q(\T^nD))\big)
  =
  \Obs(\Fact_q(\T^{n+1}D)).
\]

\begin{definition}[\textsc{Lincs}   coalgebra]
A \textsc{Lincs}   coalgebra is an INC object \(A\in\mathsf{INC}(\C)\) equipped with a structure map
\[
\gamma:A\to \T_{\INC}A
\]
that transports base INC data to its next tangent lift. In concrete learning models, \(\gamma\) may be induced by a vector field, update rule, policy perturbation, local flow, or infinitesimal repair map.
\end{definition}

\begin{definition}[Coalgebraic \textsc{Lincs}   fixed point]
A coalgebraic \textsc{Lincs}   fixed point is a \textsc{Lincs}   coalgebra \((A,\gamma)\) for which the comparison
\[
A\longrightarrow \T_{\INC}A
\]
is stable in the chosen categorical sense: isomorphism, equivalence, bisimulation, or convergence under an application-specific scalarization. At such a point, applying the tangent lift produces no essentially new INC data.
\end{definition}

For practical learning systems, the most useful notion of stability is likely not literal isomorphism, but enriched convergence: the INC tower stabilizes when the obstruction objects at successive levels are related by an approximate bisimulation and a scalar Lyapunov functional decreases to a fixed tolerance. For a neural network, reaching \(\nu\T_{\INC}\) means that further tangent unfolding does not reveal new out-of-distribution directions, declared interaction defects, or higher-order factorization obstructions not already represented by the learned tangent model.

This formulation says that \textsc{Lincs}   can be studied as the stable behavior of iterated infinitesimal non-compositionality. The base category supplies learning diagrams; the tangent functor unfolds their infinitesimal defects; the sketch-specific interaction signature tests whether generated directions remain coherent and internally expressible; and the coalgebraic fixed point is the limit at which these unfoldings become self-consistent.

\begin{definition}[Set-based class realization]
A set-based class realization of the INC endofunctor consists of a fully
faithful semantics
\[
U:\mathsf{INC}(\C)\longrightarrow\mathbf{Class}
\]
and a set-based endofunctor \(F:\mathbf{Class}\to\mathbf{Class}\) together
with a natural isomorphism
\[
U\T_{\INC}\cong FU.
\]
Here set-based means that, for every class \(A\) and every \(x\in F(A)\),
there are a set \(A_0\subseteq A\) and \(x_0\in F(A_0)\) such that
\[
x=F(i_{A_0,A})(x_0),
\]
where \(i_{A_0,A}:A_0\hookrightarrow A\) is inclusion.  We say that the
realization \emph{creates the final carrier} when the carrier and structure
map of the final \(F\)-coalgebra lie in the essential image of \(U\), and
coalgebra morphisms between realized objects are reflected by \(U\).
\end{definition}

\begin{theorem}[Existence of a final INC coalgebra]
Let \(\T_{\INC}\) admit a set-based class realization.  Then the realizing
endofunctor \(F\) has a final coalgebra in \(\mathbf{Class}\).  If the
realization creates the final carrier, then \(\T_{\INC}\) has a final
coalgebra in \(\mathsf{INC}(\C)\).
\end{theorem}

\begin{proof}
The first statement is the Final Coalgebra Theorem of Aczel and Mendler:
every set-based endofunctor on the category of classes has a final coalgebra
\citep{aczel1989final}.  Let \((Z,\zeta:Z\to FZ)\) be this final coalgebra.
When the realization creates the final carrier, choose
\(A\in\mathsf{INC}(\C)\) with \(UA\cong Z\).  The natural isomorphism
\(U\T_{\INC}\cong FU\) transports \(\zeta\) to a
\(\T_{\INC}\)-coalgebra structure on \(A\).  For any other realized
coalgebra, finality of \((Z,\zeta)\) gives a unique coalgebra morphism into
\(Z\); fullness and reflection lift it uniquely to \(A\).  Hence \(A\) is
final in \(\mathsf{INC}(\C)\).\qedhere
\end{proof}

\begin{definition}[Accessible set realization]
An accessible set realization of \(\T_{\INC}\) consists of a fully faithful
functor
\[
V:\mathsf{INC}(\C)\longrightarrow\mathbf{Set},
\]
an accessible endofunctor \(G:\mathbf{Set}\to\mathbf{Set}\), and a natural
isomorphism \(V\T_{\INC}\cong GV\).  As above, the realization
\emph{creates the final carrier} when the final \(G\)-coalgebra lies in the
essential image of \(V\) and coalgebra morphisms between realized objects are
reflected by \(V\).
\end{definition}

\begin{theorem}[Set-sized final INC coalgebra]
If \(\T_{\INC}\) admits an accessible set realization, then its realizing
endofunctor \(G\) has a final coalgebra in \(\mathbf{Set}\).  If the
realization creates the final carrier, this coalgebra lifts to a final
\(\T_{\INC}\)-coalgebra in \(\mathsf{INC}(\C)\).

More quantitatively, let \(\kappa>\aleph_0\) be regular and suppose that
\(2^\lambda\leq\kappa\) for every \(\lambda<\kappa\).  If \(G\) is
\(\kappa\)-accessible and
\[
|A|<\kappa\quad\Longrightarrow\quad |G A|\leq\kappa,
\]
then the final \(G\)-coalgebra has cardinality at most \(\kappa\).
\end{theorem}

\begin{proof}
Barr proves that the forgetful functor from coalgebras of any accessible
endofunctor on \(\mathbf{Set}\) has a right adjoint, and therefore that its
coalgebra category has a terminal object \citep{barr1993terminal}.  His
regular-cardinal refinement gives the stated bound of \(\kappa\) on the
terminal carrier under the displayed hypotheses.  The lifting from
\(\mathbf{Set}\) to \(\mathsf{INC}(\C)\) is the same fully faithful transfer
used in the preceding theorem.\qedhere
\end{proof}

\begin{conjecture}[\textsc{Lincs}  --final-coalgebra comparison]
Under the hypotheses of either preceding existence theorem, the universal
\textsc{Lincs}   completion is equivalent to the resulting final INC coalgebra:
\[
\LINCS(\C)
  \simeq
  \nu \T_{\INC},
\]
with the equivalence compatible with the embedding of base learning models and
their iterated tangent factorization problems.
\end{conjecture}

The Aczel--Mendler theorem discharges existence once the INC endofunctor has a
set-based class realization.  Barr's theorem avoids proper classes whenever
\(\T_{\INC}\) instead has an accessible realization on sets, and its
regular-cardinal form bounds the size of the final carrier.  What remains
conjectural is the comparison: that the final coalgebra so obtained also
satisfies the free-completion universal property proposed for
\(\LINCS(\C)\).  Independently, convergence of the concrete INC tower can be
proved under explicit contractivity hypotheses by metric coinduction, which
establishes properties of a limit from invariance under a contractive
approximation step \citep{kozen2009metric}.

\begin{definition}[Metric realization of INC]
A metric realization of an INC fiber is a complete metric space
\((M_D,d_D)\) of obstruction states associated with an admissible model \(D\),
together with a map
\[
F_D:M_D\to M_D
\]
realizing the action of \(\T_{\INC}\).  It is \(\rho\)-contractive when there
is a constant \(0\leq \rho<1\) such that
\[
d_D(F_Dx,F_Dy)\leq \rho d_D(x,y)
\qquad\text{for all }x,y\in M_D.
\]
The realization is \emph{bisimulation-conservative} when distance zero
identifies precisely the INC states that are bisimilar in the obstruction
fiber.
\end{definition}

\begin{theorem}[Contractive coalgebraic stabilization]
Let \((M_D,d_D,F_D)\) be a \(\rho\)-contractive metric realization of an INC
fiber, and let \(x_0\in M_D\) represent the base obstruction
\(\Obs(\Fact_q(D))\).  Define the realized INC tower by
\(x_{n+1}=F_Dx_n\).  Then:
\begin{enumerate}[leftmargin=2em]
  \item there is a unique fixed point \(x_\ast=F_Dx_\ast\);
  \item the tower converges geometrically, with
  \[
  d_D(x_n,x_\ast)
    \leq
    \frac{\rho^n}{1-\rho}\,d_D(x_1,x_0);
  \]
  \item if an approximate tower satisfies
  \(d_D(\widetilde x_{n+1},F_D\widetilde x_n)\leq\epsilon\) at every stage,
  then
  \[
  d_D(\widetilde x_n,x_\ast)
    \leq
    \rho^n d_D(\widetilde x_0,x_\ast)
    +\frac{1-\rho^n}{1-\rho}\,\epsilon,
  \]
  and hence
  \(\limsup_n d_D(\widetilde x_n,x_\ast)\leq\epsilon/(1-\rho)\).
\end{enumerate}
If, in addition, the realization is bisimulation-conservative, then the exact
tower stabilizes at a unique INC behavior up to bisimulation.  If this behavior
lifts to a final \(\T_{\INC}\)-coalgebra in the obstruction fibration and the
universal \textsc{Lincs}   embedding is dense and preserved by \(\T_{\INC}\), then it
realizes the comparison
\[
\LINCS(\C)\simeq\nu\T_{\INC}.
\]
\end{theorem}

\begin{proof}
This is the contractive fixed-point construction underlying metric coinduction
\citep{kozen2009metric}, specialized to the realized INC tower.  
For \(m>n\), contractivity gives
\[
d_D(x_{k+1},x_k)
  \leq
  \rho^k d_D(x_1,x_0).
\]
The triangle inequality therefore yields
\[
d_D(x_m,x_n)
  \leq
  \sum_{k=n}^{m-1}\rho^k d_D(x_1,x_0)
  \leq
  \frac{\rho^n}{1-\rho}d_D(x_1,x_0).
\]
Thus \((x_n)\) is Cauchy.  Completeness supplies a limit \(x_\ast\), and the
Lipschitz continuity of \(F_D\) gives
\(F_Dx_\ast=\lim_nF_Dx_n=\lim_nx_{n+1}=x_\ast\).  If \(y_\ast\) is another
fixed point, then
\[
d_D(x_\ast,y_\ast)
  =d_D(F_Dx_\ast,F_Dy_\ast)
  \leq\rho d_D(x_\ast,y_\ast),
\]
so \(d_D(x_\ast,y_\ast)=0\), proving uniqueness.  Letting \(m\to\infty\) in
the Cauchy estimate proves the geometric bound.

For the approximate tower,
\[
\begin{aligned}
d_D(\widetilde x_{n+1},x_\ast)
&\leq d_D(\widetilde x_{n+1},F_D\widetilde x_n)
   +d_D(F_D\widetilde x_n,F_Dx_\ast)\\
&\leq \epsilon+\rho d_D(\widetilde x_n,x_\ast).
\end{aligned}
\]
Induction solves this recurrence and gives the stated finite-\(n\) and
asymptotic bounds.  Bisimulation-conservativity turns uniqueness at distance
zero into uniqueness up to bisimulation.  Under the final lifting and density
hypotheses, the lifted fixed behavior has the universal property of
\(\nu\T_{\INC}\), while density and preservation extend the comparison from
presentable learning models to their \textsc{Lincs}   completion.  This proves the final
conditional statement.\qedhere
\end{proof}

The scalar certificate used in applications fits this theorem when the
obstruction energy
\[
E_n
  =
  \Phi\big(\Obs(\Fact_q(\T^nD))\big)
  +
  \sum_{k\in\Omega_D}\mu_{k,n}\Phi^k_n
  +
  \eta\Phi^{\mathrm h}_n
\]
is a Lyapunov upper bound for \(d_D(x_n,x_\ast)\).  In finite-sample learning,
an inequality \(E_{n+1}\leq\rho E_n+\epsilon\) then gives
\[
E_n\leq\rho^nE_0+\frac{1-\rho^n}{1-\rho}\epsilon,
\]
which is the scalar counterpart of approximate coalgebraic stabilization.
The theorem therefore proves the convergent fixed-point part of the
conjecture.  Establishing finality without the lifting hypotheses, and proving
that the resulting final coalgebra satisfies the free-completion universal
property, remain open.

\section{\textsc{Lincs}   as a Category}

We now collect the preceding ideas into an explicit categorical picture.

\begin{definition}[Learning diagram category]
A learning diagram category is a category \(\mathsf{Learn}(\C)\) whose objects are admissible models \(D:J\to\C\) of learning sketches, and whose morphisms are model transformations preserving the intended learning semantics.
\end{definition}

\begin{definition}[Factorization fibration]
A factorization fibration over \(\mathsf{Learn}(\C)\) is a functor
\[
p:\mathsf{Fact}(\C)\to\mathsf{Learn}(\C)
\]
whose fiber over \(D:J\to\C\) is the factorization problem \(\Fact_q(D)\) associated with its learning sketch.
\end{definition}

\begin{definition}[INC fibration]
The INC fibration is the tangent-lifted factorization fibration
\[
p_T:\mathsf{INC}(\C)\to\mathsf{Learn}(\T\C),
\]
whose fiber over \(\T D\) is the tangent factorization problem \(\Fact_q(\T D)\).
\end{definition}

\begin{definition}[Interaction-profile fibration]
When admissible learning models carry signatures $\Omega_D$, the
interaction-profile fibration is a functor
\[
p_\Omega:\mathsf{IntINC}(\C)\to\mathsf{Learn}(\T\C),
\]
whose fiber over $\T D$ is the declared closure, factorization, or profile
problem $\Obs_{\Omega,D}(\T D)$.  Connection-equipped fibers include the
ordered profile $J^2_\nabla$ and its symmetric and antisymmetric projections.
\end{definition}

\begin{definition}[Bracket-closure fibration]
When $[,]\in\Omega_D$, the bracket-closure specialization is a functor
\[
p_{[,]}:\mathsf{BrINC}(\C)\to\mathsf{Learn}(\T\C),
\]
whose fiber over \(\T D\) is the closure problem
\(\Cl_{[,]}(\mathcal A_D)\).
\end{definition}

\begin{definition}[Conditional interaction transport]
An interaction transport is a functor or, when transport is only partially
defined, a fibered correspondence over learning diagrams
\[
B_\Omega:\mathsf{Fact}(\C)\rightsquigarrow\mathsf{IntINC}(\C)
\]
such that the fiber over a base model $D$ relates the base factorization
problem $\Fact_q(D)$ to the declared interaction profile generated by its
tangent lift.  It is admissible only under the naturality, presentation
descent, and sketch-fidelity requirements of conditional
non-compositionality transport.  Taking $\Omega_D=\{[,]\}$ recovers a
bracket-transport functor into $\mathsf{BrINC}(\C)$.
\end{definition}

A \textsc{Lincs}   category can then be seen as a category of learning diagrams equipped with both a factorization fibration and its tangent lift:
\[
\mathsf{Learn}(\C)
\leftarrow
\mathsf{Fact}(\C),
\qquad
\mathsf{Learn}(\T\C)
\leftarrow
\mathsf{INC}(\C).
\]
With interaction transport, the picture refines to
\[
\mathsf{Fact}(\C)
\xrightarrow{\ B_\Omega\ }
\mathsf{IntINC}(\C)
\to
\mathsf{Learn}(\T\C).
\]
Coalgebraically, the same category carries an unfolding operation
\[
\T_{\INC}:\mathsf{INC}(\C)\to\mathsf{INC}(\C),
\]
so that \textsc{Lincs}   learning can also be viewed as constructing coalgebras
\[
A\to\T_{\INC}A
\]
whose iterated tangent unfoldings stabilize. Learning is a section-selection or repair problem: choose morphisms, parameters, coalgebra maps, or extensions that make base and tangent factorization problems closer to being inhabited, make the induced directions coherent under the sketch-specific interaction signature, and make the INC tower closer to a fixed point.  A computational realization may still assign zero admission weight to an available interaction signal.

\section{Interaction Profiles and Higher INC}

The factorization definition of INC captures first-order infinitesimal
coherence: does the tangent-lifted diagram satisfy the same universal
factorization property as the base diagram?  Second-order structure is richer
than one distinguished projection.  Lie brackets test order-sensitive
noncommutativity and involutivity.  A connection-equipped realization also
retains the ordered derivatives $\nabla_XY$ and $\nabla_YX$, whose symmetric
sum measures joint acceleration relative to that connection.

In a tangent category, vector fields and their Lie brackets can be defined from
tangent structure. Thus, once a learning model $D$ determines an admissible
distribution $\mathcal A_D$, \textsc{Lincs} can ask whether $\mathcal A_D$ is
closed under bracket. This is the tangent analogue of asking whether a
distribution is involutive.  Infinitesimal Causality models interventions as
tangent directions and studies their closure and noncommutativity
\citep{mahadevan2026infinitesimal}; BRIDGE/SKFM uses Lie-bracket geometry to
expose latent confounded causal structure \citep{mahadevan2026bridge}.  These
applications justify bracket closure as an important specialization, but not
as the unique second-order observation for every learning sketch.

\begin{definition}[Sketch-relative interaction transport]
Let $D:J\to\C$ be an admissible model with base factorization problem
$\Fact_{\mathbb S}(D)$ and interaction signature $\Omega_D$.  An interaction
transport for $D$ is a natural, presentation-descending assignment
\[
\mathfrak t_{D,\mathbb S}:
\Obs\big(\Fact_{\mathbb S}(D)\big)
\rightsquigarrow
\Obs_{\Omega,\mathbb S}(\T D).
\]
It maps base non-compositionality to a candidate sketch-specific interaction
profile.  It does not assert that every component of that profile identifies
the base repair or has positive predictive value after scalarization.
\end{definition}

\begin{definition}[Lie-bracket INC]
The Lie-bracket INC of $D$ relative to $\mathcal A_D$ is the obstruction
\[
\INC_{[,]}(D)
  =
  \{\, [X,Y]\mid X,Y\in\mathcal A_D,\ [X,Y]\notin\mathcal A_D\,\}.
\]
The model is bracket-compositional relative to $\mathcal A_D$ when
$\INC_{[,]}(D)$ is empty.
\end{definition}

\begin{definition}[Connection-profile INC]
For a connection-equipped realization, connection-profile INC is the
obstruction profile obtained from
\[
J^2_{\nabla,D}(X,Y)
=
\big(\nabla_XY,\nabla_YX\big),
\]
together with any sketch-declared factorization, closure, or observation maps.
Its symmetric and antisymmetric projections may be scalarized separately.
Neither a nonzero symmetric acceleration nor a nonzero bracket is by itself a
defect: defect status is relative to the constraints and admissible repair
distribution declared by the sketch.
\end{definition}

The distinction matters computationally.  The bracket is invariant under
coordinate change for projectable fields, whereas the symmetric projection
requires the declared connection.  Conversely, order-sensitive torque need
not be the mechanism that predicts future failure: accumulated joint
acceleration may carry more information in a particular sketch.  A validated
realization therefore admits the two components separately, for example with
sparse coefficients $(\mu_s,\mu_a)$, and permits either coefficient to vanish.

This gives a hierarchy:
\[
\begin{aligned}
&\text{base factorization problems}
\to
\text{tangent factorization problems}
\to
\text{interaction profiles}\\
&\to
\text{relative closure, curvature, and torsion}
\to
\text{higher jets}.
\end{aligned}
\]
Higher-order \textsc{Lincs} studies this hierarchy as successive,
signature-relative refinements of compositionality.

\section{Homotopical \textsc{Lincs}  }

Strict INC measures failure of strict commutativity. In many learning systems, however, exact equality of diagrams is too rigid. A representation, policy, world model, or reasoning graph may be correct only up to a controlled deformation. This suggests a homotopical refinement of \textsc{Lincs}  .

Strict repair asks for an exact factorization:
\[
\text{strict repair:}\qquad D=\bar Dq.
\]
Homotopy repair asks for a coherent deformation from the learned diagram to a factored diagram:
\[
\text{homotopy repair:}\qquad D\simeq \bar Dq.
\]
Thus DB minimizes strict diagram defects, \textsc{Lincs}   minimizes tangent-lifted defects, and homotopical \textsc{Lincs}   asks whether those tangent defects are null-homotopic, deformable, or obstructed.

\begin{definition}[Homotopical INC]
In a homotopical, model-categorical, or \(\infty\)-categorical enrichment of \(\C\), the homotopical infinitesimal non-compositionality of \(D\) is
\[
\INC_{\mathrm h}(D)
  =
  \Obs\big(\Fact_q(\T D)\ \text{up to homotopy}\big).
\]
It vanishes when the tangent compositionality problem is inhabited up to coherent homotopy.
\end{definition}

Interaction profiles and homotopy describe complementary layers of path dependence. The bracket is the local, infinitesimal shadow of noncommuting flows; connection-dependent symmetric acceleration records joint local bending; homotopy is the global coherence class of finite deformations. This yields an obstruction-theoretic ladder:
\[
\text{defect}
\to
\text{tangent defect}
\to
\text{interaction or closure defect}
\to
\text{homotopy class or cohomological obstruction}.
\]

\begin{center}
\begin{tabular}{lll}
\toprule
Layer & Object & Defect\\
\midrule
Base & category & non-commuting diagram\\
Tangent & tangent category & infinitesimal non-composition\\
Interaction & fields with signature/connection & relative non-closure or jet defect\\
Homotopy & model category or \(\infty\)-category & deformation obstruction\\
\bottomrule
\end{tabular}
\end{center}

\section{Future Directions}

The results above isolate several concrete problems whose resolution would turn
\textsc{Lincs}   from a foundations framework into a mature categorical theory of
learning.

\paragraph{Identifying the \textsc{Lincs}   completion.}
The central open problem is the \textsc{Lincs}  --final-coalgebra comparison.  The
Aczel--Mendler and Barr theorems establish existence of final INC coalgebras
under class-based or accessible set-based realizations, while the contractive
stabilization theorem controls a metric realization of the same unfolding.
What remains is to construct the canonical comparison
\[
\LINCS(\C)\longrightarrow\nu\T_{\INC}
\]
and give conditions under which it is fully faithful and essentially
surjective.  A promising route is to show that admissible learning sketches
form an accessible category, that \(\T_{\INC}\) preserves the relevant
filtered colimits, and that the presentable objects generate both sides.  Such
a result would connect the free-completion and final-semantics viewpoints
rather than treating them as parallel descriptions.

\paragraph{Intrinsic criteria for stabilization.}
The contractive theorem assumes a complete metric realization and a constant
\(\rho<1\).  An important next step is to derive these hypotheses from the
learning sketch and tangent structure themselves.  This requires identifying
when scalarizations of factorization, interaction-profile, closure, and homotopy obstructions define
a complete behavioral metric and when tangent lift is contractive or eventually
contractive in that metric.  Stochastic learning further calls for versions of
metric coinduction with martingale noise, biased tangent estimators, and
data-dependent contraction factors.  These results would turn stabilization
from an assumed analytic property into a verifiable certificate for a learning
system.

\paragraph{Higher-order obstruction theory.}
First-order INC tests tangent factorization, whereas interaction profiles
detect relations among generated infinitesimal directions.  These should not
be identified: tangent factorization is the primary first-order condition,
while Lie-bracket nonclosure is a derived, antisymmetric second-order
obstruction.  With a declared connection, the full ordered pair
$(\nabla_XY,\nabla_YX)$ and its symmetric acceleration must be retained before
selecting a projection.
A higher theory should construct prolonged learning sketches
\(\mathbb S^{(n)}\) whose models are \(\T^nD\) equipped not only with the
original factorization constraints, but also with the canonical flips,
vertical lifts, and coherence maps of iterated tangent structure.  The
corresponding hierarchy would take the form
\[
\INC^{(n)}(D)
  =
  \Obs\!\left(\Fact_{\mathbb S^{(n)}}(\T^nD)\right),
\qquad n\geq 0.
\]
At level \(\T^2\), the canonical flip and any declared connection should
organize both symmetric and antisymmetric infinitesimal interactions,
including acceleration, bracket, torsion, and curvature profiles.  At level
\(\T^3\), the relevant coherence conditions should
include Jacobiator- and Bianchi-type obstructions; at still higher levels,
jets and higher or \(L_\infty\)-style brackets may provide the appropriate
language.  In a homotopical or \(\infty\)-categorical setting, one should then
determine when these local higher defects integrate to global deformation
obstructions and when they vanish up to coherent homotopy.  Such a prolongation
theory would organize higher INC signals structurally, rather than as an
unstructured collection of repeated derivative penalties.

\paragraph{Reflective diagrammatic backpropagation.}
The coalgebraic tower also suggests a ``meta'' form of DB in which the learning
mechanism becomes part of the diagram being tested.  Let
\(\mathcal B\) denote a DB procedure that sends a global obstruction to a
compatible family of local obstruction signals and then to parameter or
architectural repairs.  A reflective DB sketch would include \(\mathcal B\),
its localization maps, and its repair maps as morphisms.  Its commutativity
conditions would ask, for example, whether localizing an obstruction and then
repairing its local pieces agrees with repairing the global diagram and then
restricting the result.  Thus the relevant question is not anthropomorphic
self-awareness, but a precise structural one: is the backpropagation mechanism
itself compositional with respect to the learning problem it is intended to
repair?

Ordinary DB detects non-compositionality in a target computation.  First-order
reflective \textsc{Lincs}-DB would detect instability in how DB transports and repairs
that obstruction, while \(\T^2,\T^3,\ldots\) would test interactions among
repair directions and expose defects that appear only after the learner
changes itself.  A nonzero meta-obstruction could then motivate an
architectural operation---adding a computational path, module, state variable,
skip connection, or compatibility constraint---rather than only a parameter
update.  This is analogous to BRIDGE/SKFM: those methods use bracket
nonclosure to reveal latent causal directions, whereas reflective
\textsc{Lincs}-DB may use a sketch-specific interaction profile to reveal latent structure in
the learning architecture itself.  Developing the corresponding meta-sketch,
its admissible architectural edits, and a convergent coalgebraic repair
procedure is a natural subject for a separate theoretical and experimental
paper.

There is also a more speculative stabilization question.  The tower
$D,\T D,\T^2D,\ldots$ is not by itself a spectrum, so repeated tangent lift
does not automatically place \textsc{Lincs}   in stable homotopy theory.  Nevertheless,
if a homotopical realization of prolonged sketches supplies compatible
structure maps and a regime in which suspension and looping become inverse,
then stabilized INC obstructions might define spectrum-valued invariants.
This would replace the question of whether a defect vanishes at one finite
tangent order by the stable question of which obstruction class persists
under all further prolongations.  Determining whether natural \textsc{Lincs}   models,
including Transformer equivariance sketches, admit such a stabilization is an
open problem.

\paragraph{Topological \textsc{Lincs}   through the nerve.}
There is a complementary route from categorical \textsc{Lincs}   to topology.  For a
fixed learning sketch and diagram $D$, let $\mathcal R_{\INC}(D)$ denote a
category whose objects are admissible base and tangent repairs and whose
morphisms are structure-preserving transformations between them.  Its nerve
\[
N\mathcal R_{\INC}(D)_n
  =
  \operatorname{Fun}([n],\mathcal R_{\INC}(D))
\]
is a simplicial set: vertices are repairs, edges are transformations between
repairs, and higher simplices encode coherent chains of transformations.  The
geometric realization $|N\mathcal R_{\INC}(D)|$ therefore turns the repair
category into a classifying space.  This construction parallels the use of
nerves and classifying spaces to study equivalence classes of categorical
causal models in Universal Causality and its higher algebraic $K$-theoretic
extension \citep{mahadevan2023universal,mahadevan2025higherk}.

MacAdam's Weil nerve suggests an infinitesimal refinement of this construction
that should be distinguished from the ordinary simplicial nerve above
\citep[Ch.~4]{macadam2022functorial}.  For an involution algebroid $A$ in a
tangent category $\C$, the Weil nerve is a fully faithful embedding
\[
N_{\mathsf{Weil}}:
\operatorname{Inv}(\C)\longrightarrow[\mathsf{Weil}_1,\C]
\]
that records $A$ as a transverse-limit-preserving tangent functor
$V\mapsto A(V)$; Segal-like exactness conditions characterize which such
functors arise from involution algebroids.  Accordingly, if the admissible
infinitesimal repairs of $D$ assemble into an involution algebroid
$A_{\INC}(D)$ in a category-valued tangent semantics, its Weil nerve should
produce repair categories
\[
\mathcal R_{\INC}^{V}(D)
  :=N_{\mathsf{Weil}}(A_{\INC}(D))(V),
  \qquad V\in\mathsf{Weil}_1.
\]
Taking the ordinary nerve and geometric realization objectwise then gives a
Weil-indexed family of spaces
\[
V\longmapsto\left|N\mathcal R_{\INC}^{V}(D)\right|.
\]
The monoidal unit records the base repair topology, the dual-number generator
records first-order repair directions, and its Weil composites organize
higher and mixed prolongations.  Thus tangent information is retained before
passing to homotopy type, rather than being forgotten by a single classifying
space construction.

Chapter~5 of MacAdam's thesis supplies a complementary integration step.  In
the category $\mathcal W$ of Weil spaces, an infinitesimal groupoid
$\partial:\mathsf{Weil}_1^{\mathrm{op}}\to\operatorname{Gpd}(\mathcal W)$
induces an infinitesimal nerve $N_{\partial}$ with a left adjoint Lie
realization $|-|_{\partial}$, constructed by left Kan extension.  The resulting
tangent adjunction
\[
|-|_{\partial}:
\operatorname{Inv}(\mathcal W)
\;
\rightleftarrows
\;
\operatorname{Gpd}(\mathcal W)
:N_{\partial}
\]
preserves products and base spaces
\citep[Ch.~5]{macadam2022functorial}.  For topological \textsc{Lincs}, the
realization $|A_{\INC}(D)|_{\partial}$ is therefore a natural candidate for a
groupoid of finite repairs integrating the infinitesimal repair algebroid.
The adjunction unit
$A_{\INC}(D)\to N_{\partial}|A_{\INC}(D)|_{\partial}$ tests whether the
infinitesimal repair semantics is recovered after integration.  Future work
should construct a comparison between the classifying space of this realized
groupoid and $|N\mathcal R_{\INC}(D)|$.  Failure of the unit, or of that
comparison, to be an equivalence would then define an integration obstruction:
a local infinitesimal repair calculus that does not assemble into the proposed
global topology of finite repairs.

Such a topological \textsc{Lincs}   would distinguish more than the existence or scalar
cost of a repair.  Connected components could classify inequivalent repair
regimes, loops could record nontrivial cycles of transformations, and higher
homotopy or homology groups could detect coherent higher obstructions.  Applying
this construction to every prolonged category
$\mathcal R_{\INC}^{(n)}(D)$ would produce a tower of spaces associated with
$D,\T D,\T^2D,\ldots$.  A central question is whether the tangent lift induces
natural maps between these nerves and, under additional delooping or
group-completion hypotheses, whether their stabilized homotopy type supplies
the spectrum-valued INC invariants envisioned above.

\paragraph{Functorial constructions across learning domains.}
The examples in this paper suggest that Bellman consistency, Kan extension,
diagrammatic backpropagation, sheaf gluing, and causal intervention can all be
viewed as instances of the same tangent factorization pattern.  A stronger
result would specify functors between these model categories and prove that
they preserve INC, declared interaction transport, and coalgebraic stabilization.  Such
transfer theorems would clarify which guarantees are genuinely domain
independent and which depend on additive, probabilistic, or smooth enrichment.

\paragraph{Computational semantics and empirical validation.}
A practical implementation should compile a learning sketch into its base and
tangent factorization problems, verify presentation descent, generate
profile-valued observations before scalarization, and expose certificates for
declared interaction closure and successive INC stabilization.
Empirical work can then test a precise claim: tangent coherence should improve
robustness to perturbations that are invisible to the base compositionality
loss.  Appropriate evaluations should therefore report base error, tangent
error, gauge or presentation invariance, one-sided and
symmetric/antisymmetric controls, stability under stochastic aggregation,
held-out admission, and out-of-distribution consistency, rather than treating
INC as an undifferentiated regularizer.  The problem classes summarized in
Section~\ref{sec:inc-examples} provide initial domains for such tests, but
the longer-term objective is a reusable categorical scientific-computing
interface in which sketches, tangent lifts, and obstruction objects are
first-class program structures.

\section{Summary}

\textsc{Lincs}   begins with a simple observation: every learning compositionality problem has a tangent lift. But the consequences are broad. If losses are scalarizations of obstructions to factorization, then tangent losses are scalarizations of obstructions to tangent factorization. They reveal how errors move, interact, commute, fail to close, deform, and unfold under repeated tangent lift.  Parameterized realizations must first descend through their presentation redundancies, and connection-dependent interaction profiles must declare the geometry that defines them.

Diagrammatic Backpropagation introduced diagrammatic non-compositionalities as learning signals. \textsc{Lincs}   introduces infinitesimal non-compositionalities as candidate learning signals. Its coalgebraic form suggests a further refinement: learning is the search for stable tangent unfoldings of non-compositionality.  The categorical framework supplies base, tangent, and sketch-specific interaction data; a validated computational realization decides contingently which auxiliary signals to admit. The resulting framework suggests a tangent-categorical reformulation of machine learning: learn not only the function, policy, representation, or world model, but the presentation-invariant infinitesimal geometry and coalgebraic self-consistency of the factorization problems by which it is learned.

\end{document}